\documentclass{IEEEtran}

\IEEEoverridecommandlockouts                              

 \usepackage{cite}

\usepackage{amsmath,amssymb,amsfonts,amsthm}
\usepackage{graphicx}
\usepackage{algorithm,algorithmic}
\usepackage{textcomp}

\usepackage{color}

\theoremstyle{plain}
\newtheorem{theorem}{Theorem}[section]
\newtheorem*{theorem*}{Theorem}
\newtheorem{proposition}[theorem]{Proposition}
\newtheorem{lemma}[theorem]{Lemma}
\newtheorem*{lemma*}{Lemma}
\newtheorem{corollary}[theorem]{Corollary}
\newtheorem*{corollary*}{Corollary}
\theoremstyle{definition}
\newtheorem{definition}[theorem]{Definition}
\newtheorem{assumption}[theorem]{Assumption}
\theoremstyle{remark}

\usepackage{makecell} 
\usepackage{tablefootnote}
\usepackage{float}
\usepackage{graphicx}
\usepackage{adjustbox}
\usepackage{makecell} 
\usepackage{enumitem}
\usepackage{booktabs} 
\usepackage{mathtools}
\usepackage{subfigure}

\newcommand{\cL}{\mathcal{L}}

\newcommand{\bI}{\mathbf{I}}

\newcommand{\W}{\mathbf{W}}

\newcommand{\R}{\mathbb{R}}

\newcommand{\X}{\mathbf{x}}
\newcommand{\Y}{\mathbf{y}}

\newcommand{\one}{\mathbf{1}}

\newcommand{\be}{\begin{equation}}
\newcommand{\ee}{\end{equation}}
\newcommand{\ba}{\begin{array}}
\newcommand{\ea}{\end{array}}
\newcommand{\bad}{\begin{aligned}}
\newcommand{\ead}{\end{aligned}}
\newcommand{\normone}[1]{\| #1 \|_1}
\newcommand{\normtwo}[1]{\| #1 \|_2}

\newcommand{\normfro}[1]{\| #1 \|_{\text{F}}}

\newcommand{\sk}{\mathrm{skew}}

\newcommand{\sym}{\mathrm{sym}}

\newcommand{\st}{\mathrm{s.t. }}

\newcommand{\St}{\mathrm{St}}

 \newcommand{\grad}{\mathrm{grad}}

\newcommand{\Proj}{\mathrm{Proj}}

\usepackage{stfloats}

\title{\LARGE \bf
Retraction-Free Decentralized Non-convex Optimization with Orthogonal Constraints
}

\author{Youbang Sun$^{1}$, Shixiang Chen$^{2}$, Alfredo Garcia$^{3}$ and Shahin Shahrampour$^{1}$
\thanks{ This work is supported in part by NSF ECCS-2240788 and NSF ECCS-2240789 Awards.}
\thanks{$^{1}$ Y. Sun and S. Shahrampour are with the Department of Mechanical and Industrial Engineering at Northeastern University, Boston, MA 02115, USA. 
        {\tt\small email:\{sun.youb,s.shahrampour\}@northeastern.edu}.}
\thanks{$^{2}$  Shixiang Chen is with the School of Mathematical Sciences, University of Science and Technology of China (USTC), Hefei, China.
{\tt \small email: shxchen@ustc.edu.cn }.}
\thanks{$^{3}$  Alfredo Garcia is with the Department of Industrial \& Systems Engineering at Texas A \& M University, College Station, TX 77845, USA. {\tt \small email: alfredo.garcia@tamu.edu }.}
}

\begin{document}

\maketitle
\thispagestyle{empty}
\pagestyle{empty}


\begin{abstract}
In this paper, we investigate decentralized non-convex optimization with orthogonal constraints. Conventional algorithms for this setting require either manifold retractions or other types of projection to ensure feasibility, both of which involve costly linear algebra operations (e.g., SVD or matrix inversion). On the other hand, infeasible methods are able to provide similar performance with higher computational efficiency. 
Inspired by this, we propose the first decentralized version of the retraction-free landing algorithm, called
\textbf{D}ecentralized \textbf{R}etraction-\textbf{F}ree \textbf{G}radient \textbf{T}racking (DRFGT).
We theoretically prove that DRFGT enjoys the ergodic  convergence rate of $\mathcal{O}(1/K)$, matching the convergence rate of centralized, retraction-based methods. We further establish that under a local Riemannian PŁ condition, DRFGT achieves the much faster linear convergence rate. Numerical experiments demonstrate that DRFGT performs on par with the state-of-the-art retraction-based methods with substantially reduced computational overhead.
\end{abstract}

\section{Introduction}
\label{sec:introduction}

This paper focuses on the decentralized non-convex optimization problem under orthogonality constraints:
 \begin{equation}\label{eq:dist_problem}
\begin{aligned}
   \min_{x_1, \ldots, x_n \in \R^{d\times r}} &\ \  \left\{\frac{1}{n} \sum_{i=1}^nf_i(x_i)\right\} \\
    \st \ \  x_i  &\in \St(d,r)\triangleq\{x\mid x^\top x=I_r\},\ \  \forall i \in [n],\\
      x_1 &=\ldots = x_n,    
\end{aligned}
\end{equation}
where $\St(d,r)$ denotes the orthogonal constraint, also known as the Stiefel manifold, and each individual function $f_i(x_i)$ is assumed to be smooth and non-convex.
Problem \eqref{eq:dist_problem} seeks to find the minimum of a global objective function (or network function) $f(x)\triangleq \frac{1}{n} \sum_{i=1}^nf_i(x)$ while maintaining consensus across the network, which naturally appears in many machine learning applications, e.g., principal component analysis (PCA) \cite{hotelling1933analysis}, canonical correlation analysis (CCA) \cite{hotelling1935canonical}, decentralized spectral analysis  \cite{kempe2008decentralized}, low-rank matrix approximation \cite{kishore2017literature} and dictionary learning \cite{raja2015cloud}.  More recently, due to the distinctive properties of orthogonal matrices,  
orthogonality constraints and regularization methods have been used for deep neural networks \cite{arjovsky2016unitary,vorontsov2017orthogonality}, providing improvements in model robustness and stability \cite{trockman2021orthogonalizing} and adaptive fine-tuning of large language models \cite{zhang2023adaptive}.

Typically, Problem \eqref{eq:dist_problem} is numerically solved by extending the classical gradient descent (GD) to the Riemannian framework using manifold optimization algorithms. Instead of the Euclidean gradient, the Riemannian gradient is calculated, and a retraction step follows in the direction of the Riemannian gradient \cite{absil2012projection}. Many retraction-based algorithms, often referred to as ``feasible" methods, exhibit similar iteration complexity rates \cite{Boumal2016,zhang2016first} to the Euclidean problems. In many instances, however, computing retractions becomes prohibitively {\it expensive} or {\it even infeasible}, especially in high-dimensional problems. For example, performing retractions on the Stiefel manifold $\St(d,r)$ requires $\mathcal{O}(dr^2)$ algebraic operations. 
Therefore, when $r$ is large, the computation of retractions becomes the predominant factor in terms of the algorithm runtime.	Moreover, retractions typically necessitate SVD or QR operations that are not as efficient and GPU-friendly as matrix multiplications; these operations introduce bottlenecks for large-scale problems.	

To address these issues, ``retraction-free" or ``infeasible" methods have been proposed. These methods only require matrix multiplications and are especially desirable when strict feasibility is not enforced during the optimization process. In retraction-free methods, iterates do not always remain on a manifold $\mathcal{M}$ but gradually converge towards it. A prominent example of retraction-free methods is the landing algorithm \cite{ablin2022fast}, analyzed in the centralized setting. Similar to the majority of modern machine learning algorithms, when retraction-free algorithms are applied to large-scale problems, scenarios such as distributed datasets or intractable computational complexity in the centralized setting necessitate distributed and/or decentralized implementation across a set of agents in a network.
However, compared to the Euclidean setting, this task is extremely challenging partly due to the non-convexity of the Stiefel manifold in Problem \eqref{eq:dist_problem}. 

\begin{table*}[h]
\caption{A comparison of existing algorithms: DRCS and DeEPCA achieve linear rate only for special cases of Problem \eqref{eq:dist_problem} (consensus and PCA, respectively). Other competitors solve \eqref{eq:dist_problem} with sublinear rates. $L$: smoothness factor, $\kappa$: condition number, $n$: network size.}

\begin{minipage}{\textwidth}
  \centering
  \begin{adjustbox}{max width=\textwidth}
    \begin{tabular}{l c c c c c c}
    \toprule
        &DRGTA
        &DRCS
        &DeEPCA
        &DESTINY
        &DRFGT
        &DRFGT\\
              & \cite{chen2021decentralized}
          &  \cite{chen2023consensus}
          &  \cite{ye2021deepca} 
          & \cite{wang2022decentralized}
          & (\textbf{This paper})
          &  (\textbf{This paper}) \\
    \midrule
    Retraction & yes & yes & yes\textsuperscript{\ref{footnote-deepca}}   & no & no &no \\
    \midrule
    PŁ Condition & no & N/A  & implied  & no &no & yes \\
    \midrule
    Communication & multi-step & multi-step &  single-step     & single-step\textsuperscript{\ref{footnote}}  & single-step\textsuperscript{\ref{footnote}}  & single-step\textsuperscript{\ref{footnote}} \\
            \midrule
    Optimal Step size & $\mathcal{O}(1/L)$ & N/A  & N/A  & $\min\{\mathcal{O}(L^{-4}), \mathcal{O}(L^{-2}/n)\}$  & $\mathcal{O}(1/L)$  & $\mathcal{O}(\kappa^{-1/4}/L) $  \\
        \midrule
    Convergence Rate & $\mathcal{O}(1/K)$ 
    &  \makecell{ linear (consensus only) }
    & \makecell{ linear  (PCA only) }
    & \multicolumn{1}{c}{$\mathcal{O}(1/K)$} & \multicolumn{1}{c}{$\mathcal{O}(1/K)$} & \multicolumn{1}{c}{linear} \\
    \bottomrule
    \end{tabular}%
    \end{adjustbox}
\label{tab:summary_exsiting}
\end{minipage}
\end{table*}

\subsection{Contributions}

In this paper, we consider a decentralized, retraction-free algorithm to solve \eqref{eq:dist_problem}, referred to as the \textbf{D}ecentralized \textbf{R}etraction-\textbf{F}ree \textbf{G}radient \textbf{T}racking (DRFGT) algorithm. The algorithm is fully decentralized and only requires agents to communicate with their neighbors to ensure convergence with consensus. We list our main contributions as follows:

\begin{itemize}
    \item The iterates of retraction-free algorithm must stay close enough to the manifold in safety regions. We characterize the safety constraint for iterates of DRFGT to remain in the neighborhood of the Stiefel manifold for the purpose of ensuring eventual feasibility (Proposition \ref{prop:safe_dist}).
    \item We prove that for general smooth functions, the iteration complexity of obtaining an $\epsilon$-stationary point for DRFGT is $\mathcal{O}(1/\epsilon)$ (Theorem \ref{thm:ergodic}). This convergence result matches the centralized version and is the first convergence result for a decentralized version of the landing algorithm.
    \item We also establish that under a local Riemannian Polyak-Łojasiewicz (PŁ) condition on $\St(d,r)$, a much faster local linear convergence can be guaranteed (Theorem \ref{thm:distributed_convergence}). To the best of the authors' knowledge, this is the {\it first-ever linear convergence result} of any decentralized Riemannian optimization algorithm (see Table \ref{tab:summary_exsiting}).
    \item We provide numerical experiments to verify our theoretical results and compare the efficiency of our algorithm with existing retraction-based algorithms.    
\end{itemize}

\footnotetext[1]{\label{footnote-deepca}
    While the updates in \cite{ye2021deepca} do not use retraction explicitly, the computation complexity of QR-decomposition is the same as a retraction step.
}
\footnotetext[2]{\label{footnote}
    Although both algorithms use single-step updates, the step size depends on the network size $n$.
}

\section{Related Literature}
\subsection{Optimization on Manifolds}
Optimization on manifolds has gained significant attention recently, both in the theoretical and practical directions. With the introduction of \textit{retraction}, a projection mapping from the manifold tangent space back to the manifold, many Euclidean optimization algorithms have been adapted to the Riemannian setting, including gradient descent \cite{absil2008optimization, zhang2016first}, quasi-Newton methods \cite{qi2010riemannian}, and even accelerated gradient methods \cite{liu2017accelerated}.

When applied to the Stiefel manifold, the aforementioned methods produce iterates that always stay on $\St(d,r)$; however, computing retractions are often a lot more expensive than gradient calculations. Instead, many application-oriented algorithms, especially for deep learning, often rely on adding a regularizer so that the iterates only stay relatively close to $\St(d,r)$. These algorithms are easy to implement due to cheaper computation costs, but they can only approximately solve the problem with sub-optimal solutions.

Recently, as an effort for computational efficiency, many works have studied unconstrained surrogate problems for manifold constraints. Although the resulting algorithms do not enforce strict feasibility at every iteration, they can be efficiently implemented with convergence guarantees. \cite{schechtman2023orthogonal} introduced ODCGM, which uses a non-smooth penalty function. Alternatively, Fletcher's penalty method \cite{fletcher2002nonlinear} introduced a smooth function, yet solving the problem often requires second-order gradient information or the help of augmented Lagrangian methods \cite{gao2019parallelizable}.
Another approach is to view the manifold as a functional constraint. \cite{lin2022complexity, boob2024level, scutari2016parallel} focused on optimization tasks in the Euclidean space that are constrained by functional equality or inequality. Most of these approaches require hierarchical optimizations in order to derive convergence. 
Recently, some works leveraged the structure of manifolds and used a simplified Fletcher's merit function, including PenC \cite{xiao2022class}, CDF \cite{xiao2024dissolving}, and the landing algorithm \cite{ablin2022fast}, which is of great relevance to our work. 

\subsection{Decentralized Extensions of Riemannian Optimization}
Decentralized optimization have been well-studied in the Euclidean domain. The decentralized GD algorithm \cite{tsitsiklis1986distributed} and its variations \cite{nedic2010constrained, chen2021distributed} perform a local gradient step with a neighborhood averaging term. When the agents are heterogeneous, these methods require a diminishing step size in order to achieve consensus among agents. Since by using a constant step size these methods are only guaranteed to converge to a neighborhood of the optimal solution, gradient tracking style methods \cite{shi2015extra, di2016next,qu2017harnessing,sun2022centralized} have been proposed to achieve exact convergence.

Nevertheless, these methods cannot be applied to Problem \eqref{eq:dist_problem}. Due to the non-convexity of $\St(d,r)$, an average of points on the manifold is not guaranteed to be on the manifold. As a result, we need to consider the problem under the scope of Riemannian optimization. Most existing Riemannian studies are designed for specific tasks, such as the PCA problem \cite{ye2021deepca} or the consensus problem \cite{sarlette2009consensus, chen2023consensus}. The work by \cite{scutari2016parallel} addressed a parallel functional inequality constrained problem, yet the convergence is only asymptotic. Recent work by \cite{chen2021decentralized, wang2022decentralized, deng2023decentralized, wang2024decentralized, wu2023decentralized, chen2023decentralized} addressed  decentralized Riemannian optimization.
DPRGD proposed by \cite{deng2023decentralized} is close to DGD \cite{tsitsiklis1986distributed} in nature and uses a diminishing step size and a projection operator to ensure the feasibility of each iterate.
\cite{chen2021decentralized} introduced DRGTA, which uses multi-step consensus and retractions. DRCGD in \cite{chen2023decentralized} considered a conjugate gradient approach with  asymptotic convergence rates. The recent work of \cite{wang2024decentralized} focused on non-smooth composite problems based on projection, which entails the same complexity as retraction. Another approach, presented in \cite{wang2022decentralized}, introduced an infeasible decentralized method named DESTINY, founded on the exact penalty function proposed for smooth problems on $\St(d,r)$ in \cite{xiao2021solving}. This was recently extended to CDADT \cite{wang2024double}, a double tracking method that addresses the generalized orthogonal constrained problems.
In contrast to \cite{xiao2021solving}, we employ a distinct merit function, defined in \eqref{eq:merit_def}, which serves as a pivotal factor for achieving a faster convergence rate and simpler analysis under various assumptions.

\section{Preliminaries}

\subsection{Notations}\label{sec:notations}
We start with definitions and notations. The Frobenius norm of matrix $A$ is denoted as $\normfro{A}$, the spectral norm of matrix $A$ is denoted as $\normtwo{A}$, the skew and symmetric parts of a square matrix $A$ are denoted as $\sk(A) = \frac{1}{2}(A-A^\top)~\text{and}~\sym(A) = \frac{1}{2}(A+A^\top)$, respectively. The  inner product of matrices $A$ and $B$ is defined as $\langle A, B\rangle \triangleq Tr(AB^\top)$. 
We denote the identity matrix of rank $r$ as $I_r$ and denote a vector of all ones as $\one$. The Kronecker product of matrices $A$ and $B$ is denoted as $A\otimes B$. The spectral radius of matrix $A$ is denoted as $\rho(A)$. We use $[n]$ to denote the set $\{1, 2, \ldots, n\}.$ 
Given a set $\mathcal{S}$, we represent $\mathcal{D}(\mathcal{S}, \delta) \triangleq \{x \mid \text{dist}(\mathcal{S}, x) \leq \delta\}$, where $\text{dist}(\mathcal{S}, x)$ denotes  $\min_{y\in\mathcal{S}} \normfro{x-y}$. $\Proj_{\St}$ represents projection on the Stiefel manifold.

\subsection{Optimization on the Stiefel Manifold}
Let us consider the problem $\min_{x\in \St(d,r)} f(x)$, which is an optimization on the Stiefel manifold. Gradient descent on the Stiefel manifold requires adapting the standard GD with respect to the manifold's geometric constraints \cite{absil2008optimization}. In each iteration, instead of taking a step using the Euclidean gradient $\nabla f(x)$, we use the Riemannian gradient $\grad f(x)$. 
Here, we use the canonical definition of the Riemannian gradient, which maps $\nabla f(x)$ onto the tangent space of $\St(d,r)$ at point $x$. The Riemannian gradient with respect to the canonical metric \cite{edelman1998geometry} 
for the Stiefel manifold is defined as
\begin{equation}\label{eq:relative-gradient}
    \grad f(x) = \sk(\nabla f(x) x^\top)x,
\end{equation}
where we drop a multiplicative factor of $2$ for convenience following \cite{ablin2023infeasible}.

The next iterate of Riemannian GD is then calculated by a retraction, which defines how a point moves on the manifold. There are many ways to define a retraction on the Stiefel manifold, such as the traditional exponential retraction \cite{edelman1998geometry}, the projection retraction \cite{absil2012projection}, QR-based retraction \cite{absil2008optimization} and the Cayley retraction \cite{wen2013feasible}.
The convergence of the retraction-based gradient methods has been well-studied in the literature. Although  retraction-based methods are relatively well-understood, the computational burden and poor scalability of the retraction operation have motivated the development of novel approaches without the use of retractions.

In this paper, our focus is on the landing algorithm proposed in \cite{ablin2022fast}, which does not require retractions and updates each iterate with a {\it landing field} $\Lambda(x)$ as follows
\begin{equation}
    x_{k+1} = x_k -  \alpha \Lambda(x_k),
\end{equation}
where $\alpha>0$ is the step size, and the landing field is calculated as
\begin{equation}
\label{eq:landing_field}
\begin{aligned}
    \Lambda(x) \triangleq& \grad f(x) + \lambda \nabla p(x),\\
    p(x)\triangleq& \frac{1}{4} \normfro{x^\top x - I_r}^2.
\end{aligned}
\end{equation}
Note that when $x\notin \St(d,r)$, we refer to $\grad f(x)$ as the relative gradient. The direction $-\grad f(x)$ aims at minimizing the function $f(x)$ and the direction $-\nabla p(x)$ works towards the feasibility constraint. In fact, these two components are orthogonal to each other, and $\langle \grad f(x), \nabla p(x)\rangle = 0$ \cite{ablin2023infeasible}. The orthogonality implies that $\Lambda(x) = 0$ if and only if $\nabla p(x)=0$ and $\grad f(x)=0$, which coincides with the stationarity condition of $x$ on the manifold.

Although the iterates of the landing algorithm are not on the manifold, they can stay close to it due to $-\nabla p(x)$. Suppose there exists a uniform upper bound on the Riemannian gradient of $f$. Then, each iterate of the landing algorithm stays within a relatively close neighborhood of $\St(d,r)$, referred to as the ``safety region" and formally defined below.

\begin{definition}[Safety Region \cite{ablin2023infeasible}] With $\epsilon\in(0,3/4)$, we define
\begin{equation}\label{eq:safe_region}
    \St(d,r)^\epsilon \triangleq \{x \in \R^{d\times r}| \normfro{x^\top x - I_r} \leq \epsilon\}.
\end{equation}
\end{definition}
We use the above definition to address the safety properties of our proposed decentralized algorithm later in Section \ref{sec:main}.

\subsection{Technical Assumptions}

In the decentralized setting, we consider a network of $n$ agents modeled with a graph $\mathcal{G} = (\mathcal{V}, \mathcal{E})$, where the agents are represented by nodes $\mathcal{V} = [n]$ and each edge $\{i,j\} \in \mathcal{E}$ denotes a connection between agent $i$ and $j$. As such, the neighborhood of agent $i$ is defined as $\mathcal{N}_i\triangleq \{j | \{i,j\}\in \mathcal{E}\}$.
Each agent $i\in [n]$ is associated with the individual objective function $f_i : \R^{d\times r} \rightarrow \mathbb{R}$ in \eqref{eq:dist_problem}. The agents work collectively to find the optimum of the global objective function $f$. We provide the following assumptions on the network as well as the objective functions.

\begin{assumption}[Network Model]\label{assump_net}
We assume the graph $\mathcal{G}$ is undirected and connected, i.e., there exists a path between any two distinct agents $i,j \in \mathcal{V}$.
\end{assumption}
Furthermore, we use a symmetric and doubly stochastic matrix $W \in \R^{n\times n}$ to capture the communication among agents. It is easy to show from stochasticity that $W\one=\one$ and that one is an eigenvalue of $W$. Additionally, the connectivity assumption guarantees that all other eigenvalues of $W$ are strictly less than $1$ in magnitude. The second largest singular value of $W$ is denoted as $\sigma_W$.
In general, $0\leq\sigma_W<1$ describes the connectivity of $\mathcal{G}$, and smaller values of $\sigma_W$ often imply a better-connected graph. For a fully connected graph with $W=\one\one^\top/n$ we have $\sigma_W = 0$.

We also provide the following assumptions on the objective functions, all of which are considered to be standard in the literature \cite{zhang2016first}.
\begin{assumption} [Lipschitz Smoothness]
\label{assump:Lip}
We assume that all local objective functions $f_i: \mathbb{R}^{
d\times r} \rightarrow \mathbb{R}, i \in [n]$ are differentiable and $L$-smooth on $\mathbb{R}^{d\times r}$, i.e. for all $i$ and $x, y \in \mathbb{R}^{d\times r}$, we have
\begin{equation}
    f_i(y) \leq f_i(x) + \langle \nabla f_i(x), y - x\rangle + \frac{L}{2}\normfro{y-x}^2.
\end{equation}
We also assume $\max_{x\in \St(d,r)^\epsilon }\normfro{\grad f_i(x)}\leq G$, $\forall i\in [n]$.
\end{assumption}
The above assumption implies that the global objective $f$ is also $L$-smooth. We further assume that $f$ is $C^2$-smooth.

\begin{assumption} [Local Riemannian PŁ on $\St(d,r)$]\label{assump:PL}
The global objective function $f(x)= \frac{1}{n} \sum_{i=1}^nf_i(x): \mathbb{R}^{
d\times r} \rightarrow \mathbb{R}$ satisfies the local Riemannian Polyak-Łojasiewicz (PŁ) condition on the Stiefel manifold with a factor $\mu>0$ if 
\begin{equation}
    |f(x) -f_{\mathcal{S}}^*| \leq \frac{1}{2\mu} \normfro{\grad f(x)}^2,
\end{equation}
for any point $x \in \St(d,r) \cap \mathcal{D}(\mathcal{S}, {2 \delta})$, where $\mathcal{S}$ denotes the set of all local minima with a given value $f_{\mathcal{S}}^*$.
\end{assumption}

The local PŁ condition is less restrictive than other commonly studied conditions, such as geodesic strong convexity. Problems satisfying the local PŁ condition include the PCA problem and generalized quadratic problem \cite{liu2019quadratic}.  In addition, we note that although for our proposed algorithm iterates do not strictly stay on the manifold (similar to the landing algorithm), we only impose a local Riemannian PŁ condition on $\St(d,r)$ with no other assumptions on the landing field itself.

\subsection{A Smooth Merit Function and the Optimality Condition}\label{sec:merit_func}

For the convergence analysis, we introduce the following smooth merit function \cite{ablin2023infeasible}, defined with respect to the global objective function:
\begin{equation}\label{eq:merit_def}
\begin{aligned}
    \cL(x) &= f(x) + h(x) + \gamma p(x) ,\\
    \text{where} \quad h(x) &\triangleq -\frac{1}{2}\langle \sym(x^\top \nabla f(x)), x^\top x - I_r\rangle.
\end{aligned} 
\end{equation}

For $\epsilon\in(0,\frac{3}{4})$, $\gamma$ must satisfy the following relationship,
\begin{equation}\label{ineq:gamma}
\begin{aligned}
    \gamma \geq  & \frac{2}{3-4\epsilon}\left(L(1-\epsilon) + 3s + \hat{L}^2 \frac{(1+\epsilon)^2}{\lambda(1-\epsilon)}\right),\\
    \text{where}~~~~s &= \sup_{x\in \St(d,r)^\epsilon} \normfro{\sym(x^\top \nabla f(x))},~~~\text{and}\\
    \hat{L} &= \max(L, \max_{x\in \St(d,r)^\epsilon} \normfro{\nabla f(x)}).
\end{aligned}
\end{equation}
In addition, it can be verified that $\forall x\in \St(d,r)$,
\begin{align}\label{eq:nablaL}
    \nabla \cL(x)  &= \nabla f(x)-x~\sym(x^\top \nabla f(x)).
\end{align}
Some existing works \cite{schechtman2023orthogonal} have used the non-smooth merit function $\cL'(x) = f(x) + \gamma \normfro{x^\top x - I_r}$ to better capture the function value change towards the normal direction in a local neighborhood of $\St(d,r)$. However, we choose to work with the smooth merit function in \eqref{eq:merit_def}, which introduces the term $h(x)$ to enable a more fine-grained analysis for the orthogonal direction, facilitating the deduction of a new PŁ inequality. 

We now state the following proposition on the analytical properties of the merit function $\cL(x)$ and refer to \cite{ablin2023infeasible} for the detailed proof of these results.

\begin{proposition}\label{prop_L}
The merit function $\cL(x)$ satisfies the following properties.
\begin{enumerate}
    \item $\cL(x)$ is $L_\cL$-smooth on $x \in \St(d, r)^\epsilon$, with $L_{\cL} \leq L_{f+h}+ (2+3\epsilon)\gamma$, where $L_{f+h}$ is the smoothness  of $f+h$.

    \item For $\rho = \min\{\frac{1}{2}, \frac{\gamma}{4\lambda(1+\epsilon)}\}$ ($\gamma$ given in \eqref{ineq:gamma}) and $x \in \St(d, r)^\epsilon$, we have 
    $$\langle \Lambda(x), \nabla \cL (x) \rangle \geq \rho \normfro{\Lambda(x)}^2.$$ 
\end{enumerate}
\end{proposition}
The two properties of Proposition \ref{prop_L} suggest that the merit function is indeed Lipschitz smooth, and a landing step can be seen as a descent step on the merit function. Moreover, since $\langle \grad f(x), \nabla p(x)\rangle=0$, the second property shows that within the neighborhood $\St(d,r)^\epsilon$, 
if $\nabla \cL(x) = 0$, we have $\nabla p(x)=0$ and $\grad f(x)=0$.

Here, the landing field is defined as \eqref{eq:landing_field} with respect to the global function $f(x)$. Similarly, we can define the local landing field $\Lambda_i(x)$ using the local objective $f_i(x)$, which we later use in the next section.

Additionally, with Assumption \ref{assump:Lip}, it is easy to show the Lipschitz continuity of $\Lambda(x)$ and $\Lambda_i(x)$, we denote the Lipschitz continuity constant as $L_\Lambda$. In an ideal scenario, $\gamma$ and $\lambda$ are chosen to be on the same order as $L$. 
It can then be seen that $L_\Lambda, L_\cL$ are also in the same order as $L$.
For the ease of analysis, we define $$L' \triangleq \max\{\hat{L} , L_\cL, L_\Lambda\}.$$

\section{Main Results}\label{sec:main}
In this section, we first propose a decentralized algorithm that solves Problem \eqref{eq:dist_problem}. We then formulate the update as a dynamical system and study its safety and stability conditions. Next, we establish both global and local convergence results for the algorithm.

\subsection{Decentralized Retraction-Free Gradient Tracking}
Existing algorithms such as DRGTA \cite{chen2021decentralized} have introduced an auxiliary tracking sequence (similar to the Euclidean case \cite{qu2017harnessing}) to estimate the global Riemannian gradient in the manifold setting. However, the DRGTA updates require retractions and complicated consensus computations. The consensus problem on the Stiefel manifold is a highly non-trivial task in itself as shown by \cite{chen2023consensus}, and the retractions could be inefficient and computationally expensive. 

This paper seeks to develop a {\it decentralized retraction-free} algorithm, only involving matrix multiplications and working based on single-step consensus, so that its implementation is significantly easier than DRGTA. Our exact update for the \textbf{D}ecentralized \textbf{R}etraction-\textbf{F}ree \textbf{G}radient \textbf{T}racking (DRFGT) is provided in Algorithm \ref{alg:distributed_update}. It is easily verified that the algorithm is fully decentralized, and the communication complexity is no more than the Euclidean counterpart in \cite{qu2017harnessing}. We also note that apart from saving on the retraction calculations, DRFGT offers additional benefits compared to \cite{chen2021decentralized}.

\begin{itemize}
    \item In Algorithm \ref{alg:distributed_update}, the agents only require one gradient calculation per iteration, which is more efficient than previous works such as \cite{wang2022decentralized}.
    \item Apart from not requiring retraction calculations, assuming availability of the Riemannian gradient via closed-form \eqref{eq:relative-gradient}, DRFGT does not need additional projections onto the Riemannian tangent space, which is necessary in \cite{chen2021decentralized}.
    \item Also, the consensus component of DRFGT is single-step unlike \cite{chen2021decentralized,chen2023consensus} that work with multi-step consensus. This is a desired property, consistent with the Euclidean gradient tracking algorithms. 
\end{itemize}
As a result, Algorithm \ref{alg:distributed_update} is very efficient both in terms of communication and computation, thanks to the unique properties offered by the retraction-free landing update.

\begin{algorithm}[t]
	\caption{Decentralized Retraction-free Gradient Tracking} 
 \label{alg:distributed_update}
	\begin{algorithmic}[1]
		\STATE{\textbf{Input:} initial point $x_0 \in \St(d,r)^\epsilon$,    $\alpha >0 , \lambda > 0$, $\epsilon\in(0,\frac{3}{4}).$ }  
        \STATE{Set $x_{i,0}= x_0; y_{i,0}= 0; \Lambda_i(x_{i,0}) = 0$ for all agents $i \in [n]$.}
		\FOR{$k=0,1,\ldots$} 
		\FOR{$i \in [n]$} 

\STATE{
  Update $x$, $\Lambda(x)$ and tracking term $y$:
\begin{equation*}
    \resizebox{.85\hsize}{!}{$
    \begin{aligned}
    \vphantom{\sum_{j\in \mathcal{N}_i}}x_{i,k+1} =& \sum_{j\in \mathcal{N}_i} W_{ij}x_{j,k} - \alpha y_{i,k},\\
    \vphantom{\sum_{j\in \mathcal{N}_i}}\Lambda_i(x_{i,k+1}) =& \grad f_i(x_{i,k+1})  + \lambda x_{i,k+1}(x_{i,k+1}^\top x_{i,k+1} - I_r),\\
        \vphantom{\sum_{j\in \mathcal{N}_i}} y_{i,k+1} =& \sum_{j\in \mathcal{N}_i} W_{ij}y_{j,k} + \Lambda_i(x_{i,k+1}) - \Lambda_i(x_{i,k}).
    \end{aligned}
    $}
\end{equation*}
}
		\ENDFOR
		\ENDFOR
	\end{algorithmic}
\end{algorithm}

\subsection{Linear System Analysis}\label{sec:linear_system}
We now study the proposed algorithm from a dynamical system perspective. For the notation convenience, we define 
\begin{equation*}
    \X_k \triangleq \begin{bmatrix} x_{1,k}^\top, \ldots, x_{n,k}^\top \end{bmatrix}^\top.
\end{equation*}
We also write the stacked version of the communication matrix as $\mathbf{W} \triangleq W \otimes I_d$. In addition, we denote the Euclidean average  of the iterate over the network as
\begin{equation*}
    \Bar{x}_k \triangleq \frac{1}{n}\sum_{i\in [n]} x_{i,k}, \quad \Bar{\X}_k \triangleq \one  \otimes \Bar{x}_k.
\end{equation*}

With the above notations, we can write Algorithm \ref{alg:distributed_update} in a matrix format as 
\begin{equation}\label{eq:matrix_update}
    \begin{aligned}
        \X_{k+1} =& \W \X_{k} - \alpha \Y_k,\\
    \Y_{k+1} =& \W \Y_{k} + \mathbf{\Lambda}_{k+1} - \mathbf{\Lambda}_k,
    \end{aligned}
\end{equation}
where $\Y_{k}$ and $\mathbf{\Lambda}_k$ are  stacked matrices, defined similar to $\X_{k}$.

For analyzing the dynamics of the distributed system, we need to identify an $\epsilon$-safety region similar to the centralized landing algorithm. In the decentralized setting, we evenly split the margin of safety into two parts: (i) the distance of the average iterate from the manifold, which we define safe as $\Bar{x} \in \St(d,r)^\frac{\epsilon}{2}$; (ii) the consensus safety, which we define safe when $\normfro{\X-\Bar{\X}} \leq \frac{\epsilon}{2}$. Combining the two, we have the following proposition on the system safety over the network.

\begin{proposition}[Safe Step Size in Networks]
\label{prop:safe_dist}
    Let $\Bar{x}_k \in \St(d,r)^\frac{\epsilon}{2}$ and $\normfro{\X_k - \Bar{\X}_k} \leq \frac{\epsilon}{10}$. Given $\normfro{\grad f_i(x)} \leq G$ for $x\in \St(d,r)^\epsilon$ and $\lambda > 0$, if the step size $\alpha$ satisfies
\begin{equation*}
\begin{aligned}
    \alpha_{safe} \triangleq& \min\bigg\{
    \frac{ (1-\sigma_W)^2\epsilon }{20 \sqrt{n}(G+\lambda \epsilon (1+\epsilon))}, \frac{\lambda \epsilon^2 (1 - \sigma_W)^2}{16 L'(G+\lambda \epsilon (1+\epsilon))},\\
    &\frac{1}{2\lambda}, \frac{\lambda \epsilon (1-\epsilon)}{2(G^2 + \lambda^2 (1+\epsilon)\epsilon^2 + \frac{\epsilon^4 \lambda^2}{16})}
    \bigg\},
\end{aligned}
\end{equation*}
the next iterate satisfies $\Bar{x}_{k+1} \in \St(d,r)^\frac{\epsilon}{2}$ and $\normfro{\X_{k+1} -  \Bar{\X}_{k+1}} \leq \frac{\epsilon}{10}$. 
\end{proposition}

Proposition \ref{prop:safe_dist} shows that with a sufficiently small step size $\alpha$, Algorithm \ref{alg:distributed_update} will stay within the $\epsilon$-safety region for any network structure and objective function.

With the safety constraint satisfied for Algorithm \ref{alg:distributed_update}, we next study the stability aspects of the system. Let us first introduce the following lemmas to bound the system consensus errors in both $\X$ and $\Y$.

\begin{lemma}\label{lem:bound_x}
    Let Assumption \ref{assump_net} hold with $\sigma_W$ as the second largest singular value of $W$. The consensus error on $\X$ satisfies the following inequality for $\alpha>0$,
    \begin{equation}
    \begin{aligned}
        \normfro{\X_k -  \Bar{\X}_k}^2 
        \leq& \frac{1+\sigma_W^2}{2} \normfro{ \X_{k-1}-  \Bar{\X}_{k-1} }^2 \\
        &+ \frac{1 + \sigma_W^2}{1-\sigma_W^2}\alpha^2 \normfro{\Y_{k-1} - \Bar{\Y}_{k-1}}^2. 
    \end{aligned}
    \end{equation}
\end{lemma}

\begin{lemma}\label{lem:bound_y}
    Let Assumption \ref{assump_net} hold with $\sigma_W$ as the second largest singular value of $W$. The consensus error on $\Y$ satisfies the following inequality for $\alpha>0$,
    \begin{equation}
    \begin{aligned}
        \normfro{\Y_k &-  \Bar{\Y}_k}^2 
    \leq  \\
    &\left(\frac{1+\sigma_W^2}{2} + 4L'^2\alpha^2\frac{1 + \sigma_W^2}{1-\sigma_W^2}\right)\normfro{\Y_{k-1} - \Bar{\Y}_{k-1}}^2 \\
    &+8L'^2  \frac{1 + \sigma_W^2}{1-\sigma_W^2}\normfro{\X_{k-1} -  \Bar{\X}_{k-1} }^2  \\
    &+ 4L'^2\alpha^2\frac{1 + \sigma_W^2}{1-\sigma_W^2}\normfro{  \Bar{\Y}_{k-1}  }^2.\\
    \end{aligned}
    \end{equation}
\end{lemma}

We use the previous two lemmas to characterize an LTI system with a state related to consensus errors. The following theorem identifies the stability condition for such LTI system, allowing us to analyze the convergence of Algorithm \ref{alg:distributed_update}.

\begin{theorem}[Stability Conditions]\label{thm:stable}
    Let Assumption \ref{assump_net} hold and the step size $\alpha\in(0,\alpha_{safe}]$ be safe. We can form a linear system with state $\Tilde{\xi}$, transition matrix $\Tilde{G}$ and input signal $\Tilde{u}$, 
\begin{equation}\label{eq:system}
    \begin{aligned}
        \Tilde{\xi}_k\ \leq \Tilde{G}\Tilde{\xi}_{k-1} + \Tilde{u}_{k-1},
    \end{aligned}
\end{equation}
where $``\leq"$ here denotes element-wise inequality and
\begin{align*}
    \Tilde{\xi}_k &\triangleq \begin{bmatrix}
        \normfro{\Y_k -  \Bar{\Y}_k}^2/L'\\
        L'\normfro{\X_k -  \Bar{\X}_k}^2
    \end{bmatrix},\\
    \Tilde{G} &\triangleq \begin{bmatrix}
        \frac{1+\sigma_W^2}{2} + 4L'^2\alpha^2\frac{1 + \sigma_W^2}{1-\sigma_W^2} & 8\frac{1 + \sigma_W^2}{1-\sigma_W^2}\\
        \frac{1 + \sigma_W^2}{1-\sigma_W^2}\alpha^2 L'^2&\frac{1+\sigma_W^2}{2}
    \end{bmatrix},\\
    \Tilde{u}_{k}&\triangleq \begin{bmatrix}
        4L'\alpha^2\frac{1 + \sigma_W^2}{1-\sigma_W^2}\normfro{  \Bar{\Y}_{k}  }^2\\
        0
    \end{bmatrix}.
\end{align*}
In addition, the linear system is stable, i.e., $\rho(\Tilde{G}) < 1$ if 
\begin{equation*}
    \alpha < \frac{(1-\sigma_W^2)^2}{1+\sigma_W^2} \frac{1}{16L'}.
\end{equation*}
\end{theorem}
Theorem \ref{thm:stable} shows that with a sufficiently small step size, the linear system describing consensus errors on $\X$ and $\Y$ is stable.

To study the convergence of our algorithm, we need to analyze the merit function introduced in Section  \ref{sec:merit_func}. Note that even though we focus on the decentralized Problem \eqref{eq:dist_problem}, we still work with the merit function based on the network function $f(x)$. We provide the following lemma on the  merit function, which describes the change in the value of the merit function with respect to the average of iterates.


\begin{lemma}\label{lem:decent}
    Let Assumptions \ref{assump_net} and \ref{assump:Lip} hold and the step size $\alpha\in(0,\alpha_{safe}]$ be safe. Then, the merit function \eqref{eq:merit_def}  satisfies the following inequality
    \begin{equation}
        \begin{aligned}
\cL(\Bar{x}_k)& -  \cL(\Bar{x}_{k-1})\leq    - \frac{\alpha \rho}{2 } \normfro{\Lambda(\Bar{x}_{k-1})}^2 \\
&+ \frac{\alpha C^2 L'^2}{2 \rho n} \normfro{\bar{\X}_{k-1} - \X_{k-1}}^2 + \frac{\alpha^2 L' }{2}\normfro{\bar{y}_{k-1}}^2,
        \end{aligned}
    \end{equation}
    for $\rho$ given in Proposition \ref{prop_L} and $C \triangleq \frac{3L'}{\lambda(1-\epsilon)} +2$.
\end{lemma}
Lemma \ref{lem:decent} shows that a decrease in the merit function value can be achieved for DRFGT with a suitable step size.

\subsection{Global Convergence of DRFGT}

\label{sec:conv}
In this section, we discuss the global convergence of DRFGT, building on the safety and stability guarantees discussed in the previous sections. We start by deriving the following bound on the accumulated consensus errors along the algorithm when the system \eqref{eq:system} is stable. The following is a corollary of Theorem \ref{thm:stable}.

\begin{corollary}\label{cor:sum_x}
    Given a stable step size $\alpha < \frac{(1-\sigma_W^2)^2}{1+\sigma_W^2} \frac{1}{16L'}$, the sum of consensus errors satisfies
    \begin{align*}
    \sum_k  \normfro{\X_k -  \Bar{\X}_k}^2 &\leq 
    \frac{(1 + \sigma_W^2)^2}{(1-\sigma_W^2)^4} 32\alpha^4 L'^2  \sum_k\normfro{  \Bar{\Y}_{k-1}  }^2.
\end{align*}
\end{corollary}

Using Lemma \ref{lem:decent} and Corollary \ref{cor:sum_x}, we establish one of our main results, which provides ergodic convergence results for the landing field as well as the consensus error.

\begin{theorem}[Global Convergence]\label{thm:ergodic}
    Let Assumptions \ref{assump_net} and \ref{assump:Lip} hold and the step size $\alpha\in(0,\alpha_{safe}]$ be safe. If the step size additionally satisfies
    $$\alpha< \min \left\{ \frac{(1-\sigma_W^2)^2}{1+\sigma_W^2} \frac{1}{16L'},\sqrt[3]{\frac{\rho (1-\sigma_W^2)^4}{(1 + \sigma_W^2)^2 C^2}} \frac{1}{4L'}, \frac{\rho}{8L'} \right\},$$ we get the following ergodic convergence results
\begin{equation*}
\resizebox{0.99\hsize}{!}{$
\begin{aligned}
    &\frac{\sum_k \normfro{\Lambda(\Bar{x}_{k-1})}^2 }{K} \leq \frac{1}{K} \frac{4}{\alpha \rho}(\cL(\Bar{x}_{0}) -\cL(\Bar{x}_K))  ,\\
    &\frac{\sum_k  \normfro{\X_k -  \Bar{\X}_k}^2}{K}  \leq \frac{1}{K} \frac{(1 + \sigma_W^2)^2}{(1-\sigma_W^2)^4}   \frac{512 n \alpha^3 L'^2}{ \rho}(\cL(\Bar{x}_{0}) -\cL(\Bar{x}_K)).
  \end{aligned}
  $}
\end{equation*}
\end{theorem}
The decay of the landing field in Theorem \ref{thm:ergodic} shows that DRFGT converges to a stationary point of Problem \eqref{eq:dist_problem} on the Stiefel manifold with asymptotically zero consensus error. The constraint on the step size depends on multiple factors, but the step size still scales with $\mathcal{O}(1/L')$ (when $\lambda=\mathcal{O}(L')$) across these terms.
As a result, the convergence rate matches the retraction-based algorithms such as \cite{chen2021decentralized} and the Euclidean gradient tracking in \cite{qu2017harnessing}.

\subsection{Local Linear Convergence of DRFGT Under Local PŁ Condition}\label{sec:conv_PL}

In this section, we study the local convergence of  DRFGT under the assumption that the network function satisfies the local PŁ condition (Assumption \ref{assump:PL}). Given the global convergence result in the previous section, one can guarantee that the algorithm does converge to a stationary point, but if the stationary point is a local minimizer, the local convergence could be much faster. Therefore, for local analysis, we assume the iterates are already close to a local minimizer, or equivalently, we initialize the algorithm close enough to a local minimizer. This assumption is not surprising; even some classical methods (e.g., Newton's method) achieve faster local rates when initialization is close enough to a local minimizer.

We study the convergence of the algorithm using a dynamical system perspective again. 
Consider the state vector, 
\begin{equation}\label{eq:state_vector}
    \xi_k \triangleq \begin{bmatrix}
        \normfro{\Y_k -  \Bar{\Y}_k}^2/L',\ 
        L'\normfro{\X_k -  \Bar{\X}_k}^2,\ 
        n(\cL(\Bar{x}_k) - \cL_{\mathcal{S}}^*)
    \end{bmatrix}^\top,
\end{equation}
where $ \cL_{\mathcal{S}}^*$ is the merit function evaluated at the local minimizer. It is obvious  that $ \cL_{\mathcal{S}}^* =  f_{\mathcal{S}}^*$ based on \eqref{eq:merit_def}. To establish the local convergence, we need to identify the dynamical system describing the state $\xi_k$. To that end, having Lemma \ref{lem:bound_x} already in place, we further derive Lemmas \ref{lem:dist_V} and \ref{lem:bound_y_PL} to characterize the dynamical system.

First, let us analyze the safety in the sense of the local PŁ neighborhood and present the following lemma based on the relation between the PŁ condition and the quadratic growth (QG) condition \cite{rebjock2023fast}.
\begin{lemma}[PŁ-QG]\label{lem:KL-QG} 
Let Assumptions \ref{assump:Lip} and \ref{assump:PL} hold. For $x \in \St(d,r)^\epsilon \cap \mathcal{D}(\mathcal{S}, \delta)$, the merit function $\cL(x)$ satisfies
    $$\cL(x) - \cL_\mathcal{S}^\ast\geq \frac{\mu' \rho^2 }{4} \text{dist}(\mathcal{S}, x)^2.$$
\end{lemma}
The above lemma is used to characterize an upper bound on the closeness of the iterates from the local minima set, which is used to guarantee that the iterates never leave a neighborhood of the local minima under appropriate initialization. We next present the following lemma to provide a sufficient descent inequality over $\cL(\Bar{x}_k)$, which characterizes the behavior of the last entry in \eqref{eq:state_vector}.

\begin{lemma}\label{lem:dist_V}
    Given Assumptions \ref{assump:Lip} and \ref{assump:PL}, if $\Bar{x}_{k-1} \in \mathcal{D}(\mathcal{S}, \delta) \cap \St(d, r)^\epsilon$  and the safe step size satisfies $\alpha \leq \frac{\rho}{2 L'}$, the merit function evaluated at $\Bar{x}_k$ can be upper bounded using the following inequality, 
    \vspace{-2mm}
    \begin{equation*}
    \begin{aligned}
        \cL(\Bar{x}_k) - \cL_{\mathcal{S}}^* 
            \leq & (1  - { \frac{\rho \alpha \mu' }{4}})(\cL(\Bar{x}_{k-1})- \cL_{\mathcal{S}}^* ) \\
    &+ (L'^2\alpha^2 + \frac{\alpha L' C^2}{\rho} ) \frac{L'}{n}\normfro{\X_{k-1} - \Bar{\X}_{k-1}}^2.
    \end{aligned}
    \end{equation*}
\end{lemma}

In addition, since the iterates stay close enough to a local minimizer, we can revisit the result of Lemma \ref{lem:bound_y} to get the following relationship.

\begin{lemma}\label{lem:bound_y_PL}
    Let Assumptions \ref{assump_net} and \ref{assump:Lip} hold, assume $\delta \leq 1$, and select a safe step size $\alpha\in(0,\alpha_{safe}]$. Let also $\Bar{x}_{k-1} \in \mathcal{D}(\mathcal{S}, \delta) \cap \St(d, r)^\epsilon$.
    Then, the consensus error on $\Y$ satisfies the following inequality,
    \begin{equation*}
    \begin{aligned}
        \normfro{\Y_k -  \Bar{\Y}_k}^2 
        &\leq 
      \Big(\frac{1+\sigma_W^2}{2} + 4L'^2\alpha^2\frac{1 + \sigma_W^2}{1-\sigma_W^2}\Big)\normfro{\Y_{k-1} -  \Bar{\Y}_{k-1}}^2\\
      &+ \frac{96n \alpha^2L'^3}{\rho^2}\frac{1 + \sigma_W^2}{1-\sigma_W^2}(\cL(\Bar{x}_{k-1})- \cL_{\mathcal{S}}^*)\\
    &+8L'^2\frac{1 + \sigma_W^2}{1-\sigma_W^2} 
    ( 1+\alpha^2 L'^2)
    \normfro{\X_{k-1} -  \Bar{\X}_{k-1} }^2.
    \end{aligned}
    \end{equation*}
\end{lemma}
After these lemmas, we present another main result of this paper, which establishes the local linear rate of DRFGT, the first for any decentralized Riemannian algorithm.  
\begin{theorem}[Local Linear Convergence]
    \label{thm:distributed_convergence}
    Suppose that Assumptions \ref{assump_net}, \ref{assump:Lip} and \ref{assump:PL} hold. Additionally, let the step size $\alpha\in(0,\alpha_{safe}]$ be safe and satisfy 
\begin{equation}\label{eq:step}
    \resizebox{.96\hsize}{!}{$
\begin{aligned}
    \alpha \leq \min & \left\{ \frac{\rho}{2L'}, \frac{1-\sigma_W^2}{\rho\mu'} , \frac{\sqrt{1-\sigma_W^2}}{4L' \sqrt{\Theta} \sqrt[4]{1 + \frac{12\Phi}{\rho^2}}}, \frac{1-\sigma_W^2}{16  L'\Theta} \right\},
\end{aligned}
$}
    \end{equation}
    with $\Theta \triangleq \frac{1 + \sigma_W^2}{1-\sigma_W^2} ~\text{and}~
    \Phi \triangleq \frac{4 L'}{\rho\mu'} + \frac{ 8L'C^2  }{\rho^2\mu'}$. Then, the system state vector $\xi_k$ in \eqref{eq:state_vector} satisfies $\xi_k \leq M \xi_{k-1}$ element-wise for any iteration $k$ if $\Bar{x}_0 \in \mathcal{D}(\mathcal{S}, \delta) \cap \St(d, r)^{\frac{\epsilon}{2}}$ and 
    $\xi_0 \leq \frac{n \mu' \rho^2 \delta^2}{8}\mathbf{v}$ for $\delta \leq 1$, where $M$ is defined as
    \begin{equation*}
    M\triangleq\begin{bmatrix}
    \frac{1+\sigma_W^2}{2} + {4L'^2\alpha^2}\Theta &
    8 ( 1+\alpha^2 L'^2)\Theta  &
    \frac{96\alpha^2 L'^2 \Theta}{ \rho^2} \\
    \alpha^2 L'^2 \Theta&
    \frac{1+\sigma_W^2}{2}&
    0\\
    0&\alpha^2L'^2 + \frac{\alpha L' C^2 }{\rho}&  1  - \frac{\alpha \rho\mu'}{4}
    \end{bmatrix},
    \end{equation*}
    and $\mathbf{v}$ is the eigenvector corresponding to $\rho(M)$.
Furthermore, the system converges with the linear rate $\rho(M)\leq 1  - \frac{\alpha \rho \mu'}{8}$.
\end{theorem}  
\begin{figure*}[ht]
    \centering
    \centerline{\includegraphics[width=1.2\textwidth]{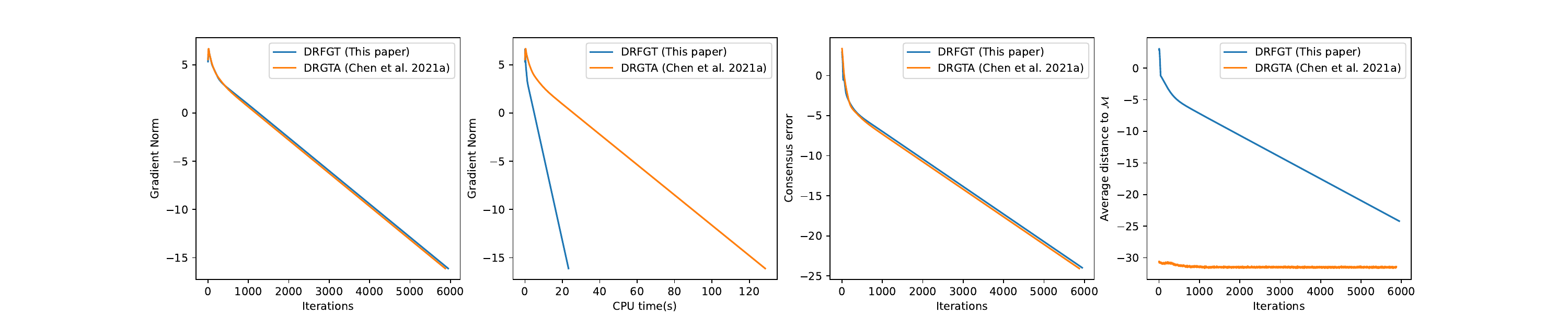}}
    \caption{Convergence of DRFGT (our work) compared to DRGTA in \cite{chen2021decentralized}.}
    \label{fig:line_compare}
\end{figure*}


We refer to the Appendix for the proof and analysis. Apart from the safety constraints on $\alpha$, to ensure that the system converges linearly, 
$\alpha$ also depends on several factors including the PŁ condition factor $\mu'$ and Lipschitz smoothness factor $L'$. 
The rate also depends on the graph connectivity, where the convergence becomes slower as $\sigma_W$ tends to one. 
The optimal step size in Euclidean GD is known to be $\alpha = \mathcal{O} (1/L).$ 
This is mostly true in our analysis as the step size is always $\mathcal{O}(1/L')$; 
however, if the dominating factor in \eqref{eq:step} is the term depending on $\Phi$, we will have $\alpha $ scaling as $\mathcal{O}(\kappa^{-1/4}/L') $, which also recovers the best previously known result on Euclidean distributed non-convex optimization \cite{xin2021improved}.
On the other hand, our safety step size scales as 
$\alpha_{safe} = \mathcal{O}(1/\sqrt{n}L')$, an improvement compared to DESTINY in \cite{wang2022decentralized}, where 
$\alpha_{safe} \leq \min\{\mathcal{O}(L^{-4}), \mathcal{O}(L^{-2}/n)\}$. Our tighter bounds on the step size are partly due to our choice of the landing field and the merit function, between which various relationships were identified (see Lemmas \ref{lem:pseudo-grad-dominate}, \ref{lem:bound_C}, \ref{lem:descent_lemma}).

\begin{figure*}[t]
    \centering
    \subfigure[Convergence vs. epoch.]{
        \includegraphics[width=0.29\textwidth]{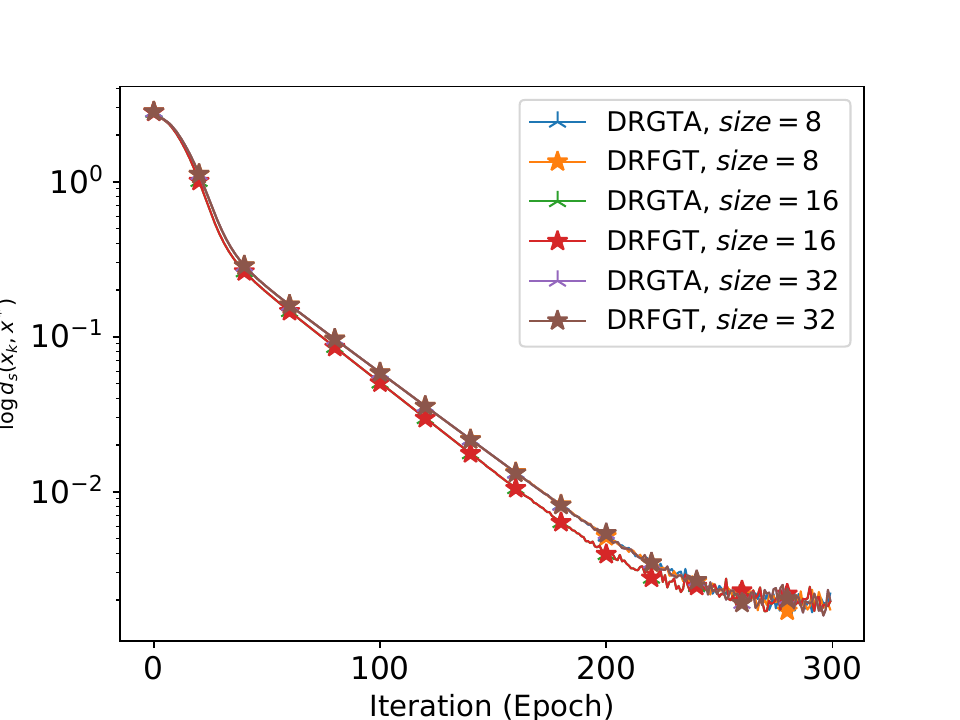}
        \label{fig:mnist_converge}
        }
    \subfigure[Convergence vs. CPU time.]{
        \includegraphics[width=0.29\textwidth]{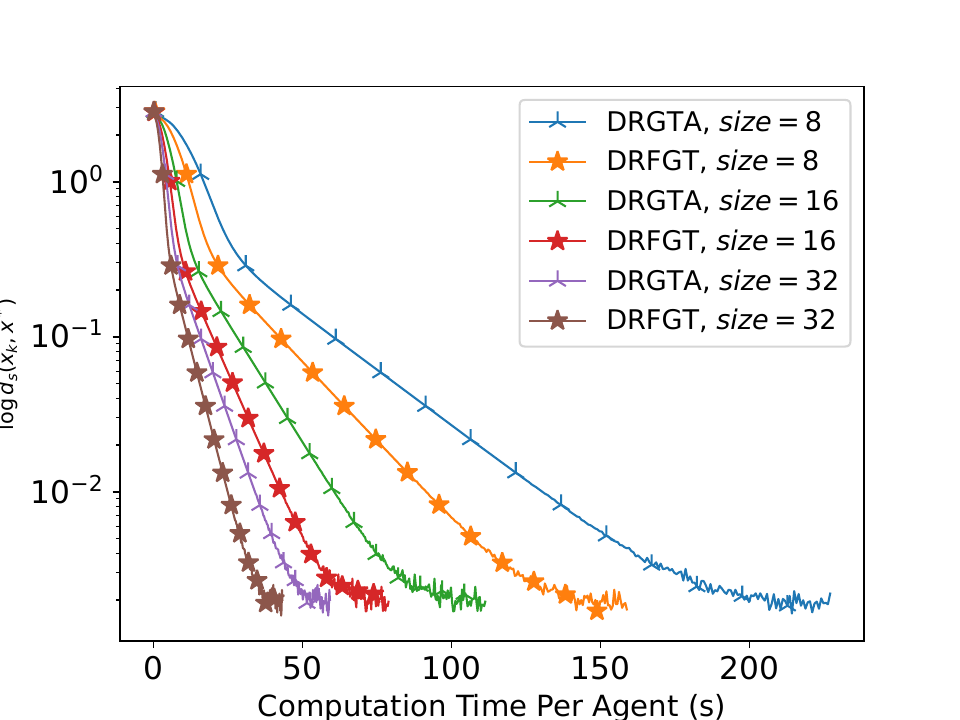}
        \label{fig:mnist_n}
        }
    \subfigure[Run time with different values of rank $r$.]{
    \includegraphics[width=0.35\textwidth]{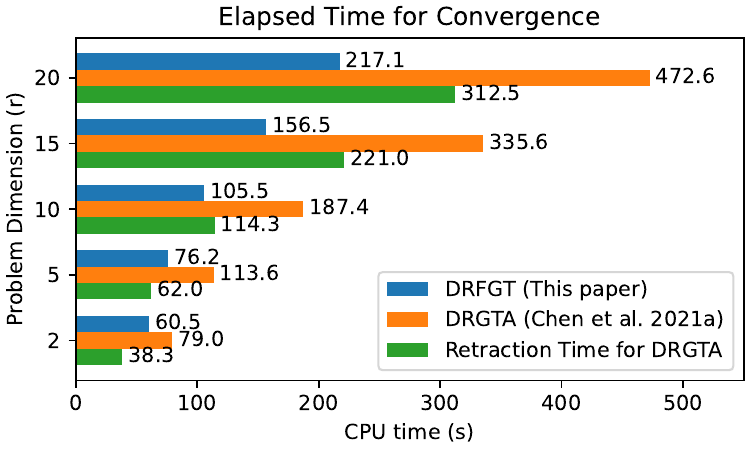}
    \label{fig:mnist_run_time}
    }
    \caption{Experiment with MNIST data.}
\end{figure*}

\section{Numerical Results}   
For the experiments, we evaluate DRFGT on a traditional PCA task with both synthetic and real-world data.

 \subsection{PCA Experiment with Synthetic Data}\label{sec:syn}
We study the decentralized PCA problem, also evaluated by \cite{chen2021decentralized}. The problem is defined as
\begin{equation}
\begin{aligned}\label{eq:numerical}
    &\min_{x_1, \ldots, x_n \in \R^{d\times r}}\left\{ \frac{1}{n}\sum_{i \in [n]} \langle A_i^\top A_i x_i,x_i D\rangle \right\} \\
    &\st \ \  x_1 = \ldots = x_n, \ \  x_i \in \St(d,r),\ \  \forall i \in [n],
\end{aligned}
\end{equation}
where $A_i$ is the individual data matrix for agent $i$, and $D$ is a  diagonal matrix with $[D]_{11}>\cdots>[D]_{rr}>0$. The dimension of the $A_i$ matrix is set to $m_i \times d$, where $d$ is the dimension of data and $m_i$ is the number of data samples assigned to agent $i$. Each entry of $A_i$ is first generated independently from a normal distribution, then the eigenvalues of $ A_i^\top A_i$ are manually adjusted to get a larger condition number to increase the problem difficulty. The network has $10$ agents, a ring structure, and $W$ has diagonal elements $W_{ii} = 0.8$ and off-diagonal elements $W_{ij} = 0.1.$
Each agent $i$ stores $m_i=m=1000$ total data points with a data dimension of $d=100$. We set the rank number $r = 10$ and the step size $\alpha=\frac{1}{10m} = 10^{-4}$. Lastly, $\lambda$ is set to be $0.1/\alpha$.

We compare the performance of Algorithm \ref{alg:distributed_update} with DRGTA, the retraction-based algorithm in \cite{chen2021decentralized} with the same initial point $x_0$, graph structure $W$, and step size $\alpha$. 
We plot the evolution of gradient norm $\normfro{\sum_{i\in [n]} \grad f_i(x_i)}$, consensus error $\sum_{i\in [n]}\normfro{x_i - \Bar{x}}$, as well as the average distance to $\St(d,r)$, $\frac{1}{n}\sum_{i\in [n]}\normfro{x_i - x_i'}$, with respect to different iterations as well as CPU time. 
The results are provided on log-scale in Fig. \ref{fig:line_compare}: (i) In terms of iterations, the DRFGT algorithm closely matches the performance of DRGTA in gradient norm and consensus error. (ii) The DRFGT algorithm converges much faster than DRGTA in CPU time and is more computationally efficient. (iii) Though DRFGT does not strictly satisfy the feasibility constraints of Problem \eqref{eq:dist_problem}, it eventually converges to a feasible critical point on $\St(d,r)$ and catches up with DRGTA.

\subsection{PCA Experiment with Real-World Data}\label{sec:real}
 We also consider the PCA task defined in \eqref{eq:numerical} on the MNIST dataset \cite{lecun1998gradient} to show the effectiveness of DRFGT. The MNIST dataset has dimension $d = 784$ and $60000$ data samples are evenly distributed across $n = \{8, 16, 32\}$ agents. We set $r = 5$ and the graph structure is set as a ring with the Metropolis constant weight matrix following \cite{chen2021decentralized}. In this task, we specifically highlight the effect of network size $n$ in the DRFGT algorithm through a multi-processing experimental setting (mpi4py).

We plot the distance to optimality $\log(d(\bar{x}, x^*))$ versus the epoch number and computation time. The results are provided in Fig. \ref{fig:mnist_converge} and \ref{fig:mnist_n}. For each specific choice of $n$, DRFGT converges significantly faster than DRGTA in CPU time. The performance with respect to epoch is similar across different $n$, which is expected. In addition, by increasing the number of agents in the network, computation time (per agent) can be accelerated by a nearly linear ratio, known as the linear speed-up effect \cite{lian2017can}. We also plot the run-time comparison for different values of $r$ in Fig. \ref{fig:mnist_run_time}, explicitly showing the computation time saved by DRFGT by avoiding retractions.

\section{Conclusion}
We proposed DRFGT to solve distributed non-convex optimization problems under orthogonal constraints and provided a safety step size guarantee to ensure DRFGT remains in the neighborhood of the Stiefel manifold. We proved that the convergence rate for DRFGT is $\mathcal{O}(1/K)$, the first convergence result for the decentralized landing algorithm. 
When the network function further satisfies a local Riemannian PŁ condition, we established a local linear convergence, the first-ever linear result among distributed Riemannian optimization algorithms.
Numerical results were provided as validation of our analysis. An interesting future direction is to extend this algorithm beyond Stiefel manifolds, similar to \cite{xiao2024dissolving}.

\section*{Appendix}
\subsection{Useful Lemmas and Inequalities}
In this section, we present some fundamental inequalities and lemmas, which are useful for the technical analysis. We first gather the following facts in one place.
\begin{lemma}[\cite{ablin2023infeasible}]
For any given $x \in \St(d,r)^\epsilon$, the singular values of $x$ are between $\sqrt{1-\epsilon}$ and $\sqrt{1+\epsilon}$.
\end{lemma}
In our proofs, we often times use the looser bounds $1-\epsilon \leq \normtwo{x}\leq 1+\epsilon$ to simplify the analysis.
\begin{lemma}\label{lemma:fact}
The landing field $\Lambda(x)$, defined in \eqref{eq:landing_field}, is Lipschitz continuous with a factor $L_\Lambda \leq L'$ and has orthogonal components, i.e., $\langle \grad f(x), \nabla p(x)\rangle=0$. The same hold for the local landing field $\Lambda_i(x)$ defined with respect to local Riemannian gradient $\grad f_i(x)$.
\end{lemma}
The proof is standard and can be found in \cite{ablin2023infeasible} for $\Lambda(x)$, and the extension to local landing field is trivial.


\begin{lemma}\label{lemma:orthogonal}
    For any given $x \in \St(d,r)^\epsilon$, $x' = \Proj_{\St}(x)$, we have the following relationship
    $$
    \langle \nabla \cL(x'), x-x'\rangle=0.
    $$
\end{lemma}
\begin{proof}
If $x=USV^\top$ is the SVD decomposition for $x$, then $x'=UV^\top$. By observing the closed-form formula for $\nabla\cL(x')$ in \eqref{eq:nablaL} the rest of the proof is simple algebra, omitted due to space limitation.
\end{proof}

We then state the following lemma, which helps divide the gap between optimality and feasibility at each iteration. 
\begin{lemma}\label{lemma:proj_dist}
For any given $x \in \St(d,r)^\epsilon$, $x' = \Proj_{\St}(x)$, we have the following inequality
    \begin{equation}\label{ineq:two_distances}
          \normfro{x-x'} \leq \normfro{x^\top x-I_r}. 
    \end{equation}
\end{lemma}
The proof can be found in Lemma 4 of \cite{chen2024tighterrorboundssignconstrained}.
We then present the following lemma, a well-known result from \cite{horn2012matrix}.

\begin{lemma}[Matrix Spectral Radius \cite{horn2012matrix}]\label{lem:spectral_rad}
Let $\mathbf{X} \in \mathbb{R}^{d\times d}$ be a non-negative matrix and $\X \in \R^d$ be a positive vector. If $\mathbf{X}\X < \X$, then $\rho(\mathbf{X})<1$. Furthermore, if $\mathbf{X}\X \leq z\X$ for some $z>0$, then $\rho(\mathbf{X}) \leq z$.
\end{lemma}
Another useful lemma is the following.
\begin{lemma}\label{lemma:aux}
For any point $x\in \St(d, r)$, we have that
\begin{equation}
    \normfro{\nabla \cL(x)} \leq 2\normfro{\Lambda(x)}.
\end{equation}
\end{lemma}
\begin{proof}
We know from \eqref{eq:nablaL} that
$$\nabla \cL(x) = (I_d - \frac{1}{2} x x^\top) \nabla f(x) - \frac{1}{2} x \nabla f(x)^\top x,$$
which implies that
\begin{align*}
    \normfro{\nabla \cL(x)}^2 
    &= Tr\Big(\nabla f(x)^\top (I_d - \frac{3}{4} x x^\top)\nabla f(x) \\
    &+ \frac{1}{4} x^\top \nabla f(x) \nabla f(x)^\top x - \frac{1}{4} \nabla f(x)^\top x\nabla f(x)^\top x \\
    &- \frac{1}{4}x^\top \nabla f(x) x^\top \nabla f(x)\Big).
\end{align*}
And for the landing field, we have 
$$\Lambda(x) = \grad f(x) = \frac{1}{2} \nabla f(x) - \frac{1}{2} x \nabla f(x)^\top x,$$
which gives
\begin{align*}
    \normfro{\Lambda(x)}^2 
    &= Tr\Big(\frac{1}{4} \nabla f(x)^\top \nabla f(x)+ \frac{1}{4} x^\top \nabla f(x) \nabla f(x)^\top x \\
    &- \frac{1}{4} \nabla f(x)^\top x\nabla f(x)^\top x - \frac{1}{4} x^\top \nabla f(x) x^\top \nabla f(x)\Big).
\end{align*}
We then have the result as
\begin{align*}
     &4 \normfro{\Lambda(x)}^2 - \normfro{\nabla \cL(x)}^2\\
    =& Tr\Big( \frac{3}{4}\nabla f(x)^\top  x x^\top\nabla f(x) +  \frac{3}{4} x^\top \nabla f(x) \nabla f(x)^\top x \\
    &- \frac{3}{4}  \nabla f(x)^\top x\nabla f(x)^\top x - \frac{3}{4} x^\top \nabla f(x) x^\top \nabla f(x) \Big)\geq 0,
\end{align*}
where the inequality is due to the fact that $Tr(AA^\top) \geq Tr(AA)$ for a square matrix $A$.
\end{proof}

\begin{lemma}[Lipschitz Continuity of $\grad f$]\label{lem:L of gradf} For $x\in \St(d,r)^\epsilon$ and $x'=\Proj_{\St}(x)$, it follows that
    \begin{equation}
        \normfro{\grad f(x)- \grad f(x')}\leq (3+2\epsilon)\hat{L} \normfro{x-x'},
    \end{equation}
    where $\hat{L}$ is defined in \eqref{ineq:gamma}. 
\end{lemma}
\begin{proof}
Based on the definition in \eqref{eq:relative-gradient}, we have
    \begin{equation}
        \begin{aligned}
            &\normfro{\grad f(x)- \grad f(x') }\\
            = &\normfro{\sk(\nabla f(x) x^\top)x - \sk(\nabla f(x') x'^\top)x'}\\
            =&\normfro{\sk(\nabla f(x) x^\top)x -\sk(\nabla f(x') x'^\top)x \\
            &+\sk(\nabla f(x') x'^\top)x - \sk(\nabla f(x') x'^\top)x'}\\
            \leq&\normfro{\sk(\nabla f(x) x^\top)- \sk(\nabla f(x') x'^\top)} \normtwo{x}\\
            &+  \normtwo{\sk(\nabla f(x') x'^\top)} \normfro{x-x'}\\
            \leq &(1+\epsilon) \normfro{\nabla f(x) x^\top-\nabla f(x') x'^\top} \\
            &+ \big(\text{max}_{x'\in\St(d,r)}\normfro{\nabla f(x')}\big)\normfro{x-x'}\\
            \leq &(1+\epsilon) (\normfro{\nabla f(x) x^\top-\nabla f(x) x'^\top} \\
            &+\normfro{\nabla f(x) x'^\top-\nabla f(x') x'^\top})+ \hat{L} \normfro{x-x'}\\
            \leq  &(3+2\epsilon)\hat{L}   \normfro{x-x'} .
        \end{aligned}
    \end{equation}
\end{proof}

The next three lemmas describe some useful relationships between the landing field $\Lambda(x)$ and merit function $\cL(x)$. First, we show a property similar to the gradient domination on the Riemannian manifold, described as follows.
\begin{lemma}[Pseudo Gradient Domination]\label{lem:pseudo-grad-dominate}
Let $x \in \St(d,r)^\epsilon \cap \mathcal{D}(\mathcal{S}, \delta) $, and without the loss of generality, assume $f_{\mathcal{S}}^* =0$. Then under Assumptions \ref{assump:Lip} and \ref{assump:PL}, we have
\begin{equation}
\begin{aligned}
    \mathcal{L}(x) &\leq \frac{1}{\mu'} \normfro{\Lambda(x)}^2,
\end{aligned}
\end{equation}
where $\ \ \frac{1}{\mu'} = \max\left\{\frac{1}{\mu}, \frac{2(3+2\epsilon)^2\hat{L}^2 + \mu L'}{2\mu \lambda^2(1-\epsilon)^2} \right\}.$
\end{lemma}
\begin{proof}
We first write the following inequality based on the smoothness of $\cL$ for $x'=\Proj_{\St}(x)$,
    \begin{equation}
    \begin{aligned}
        \mathcal{L}(x)    
        \leq& \mathcal{L}(x') + \langle \nabla \mathcal{L}(x'), x - x'\rangle + \frac{L'}{2}\normfro{x - x'}^2 \\
        =& f(x')  + \frac{L'}{2}\normfro{x - x'}^2 \\
        \leq &  f(x')   + \frac{L'}{2}\normfro{x^\top x - I_r}^2 .
    \end{aligned}
\end{equation}
where the equality is due to Lemma \ref{lemma:orthogonal}, and the last inequality follows from Lemma \ref{lemma:proj_dist}.

Since $\mathcal{S} \subseteq \St(d, r)$, $\text{dist}( x', x) \leq \text{dist}( \mathcal{S}, x)$, and with triangle inequality, $\text{dist}( \mathcal{S}, x') \leq \text{dist}( \mathcal{S}, x) + \text{dist}( x', x) \leq 2\delta$. Therefore, $x' \in \mathcal{D}(\mathcal{S}, {2 \delta})$, and we can apply Assumption \ref{assump:PL} on point $x'.$ 


Using the PŁ inequality for $x'$ and Lipschitz continuity of $\grad f(x)$ in Lemma \ref{lem:L of gradf}, we have
\begin{equation}
    \begin{aligned}
         &f(x') \leq \frac{1}{2\mu}\normfro{\grad f(x')}^2 \vphantom{\frac{(3+\epsilon)^2\hat{L}^2}{ \mu}}\\
         = &\frac{1}{2\mu} \normfro{\grad f(x) + {\grad f(x') - \grad f(x)}}^2 \vphantom{\frac{(3+\epsilon)^2\hat{L}^2}{ \mu}}\\
         \leq & \frac{1}{\mu}\normfro{\grad f(x) }^2 +  \frac{1}{\mu}\normfro{ {\grad f(x') - \grad f(x)}}^2 \vphantom{\frac{(3+\epsilon)^2\hat{L}^2}{ \mu}}\\
         \leq &\frac{1}{\mu} \normfro{\grad f(x) }^2 + \frac{(3+2\epsilon)^2\hat{L}^2}{ \mu} \normfro{ x' - x}^2\\
         \leq & \frac{1}{\mu}\normfro{\grad f(x) }^2 +\frac{(3+2\epsilon)^2\hat{L}^2}{ \mu} \normfro{x^\top x - I_r}^2,
    \end{aligned}
\end{equation}
where the last line is due to Lemma \ref{lemma:proj_dist}. Note that 
\begin{equation}\label{ineq:relation_p(x)_and_grad}
\begin{aligned}
    \normfro{\Lambda(x)}^2 \geq &\normfro{\grad f(x)}^2 + \lambda^2 (1-\epsilon)^2\normfro{x^\top x - I_r}^2.  
\end{aligned}
\end{equation}
Therefore, 
\begin{align*}
    \mathcal{L}(x) \leq &\frac{1}{\mu}\normfro{\grad f(x) }^2 +\Big(\hat{L}^2\frac{(3+2\epsilon)^2}{ \mu}+ \frac{L'}{2}\Big) \normfro{x^\top x - I_r}^2 \\
    \leq &\frac{1}{\mu'} \normfro{\Lambda(x)}^2,
\end{align*}
where $$\frac{1}{\mu'} = \max\left\{\frac{1}{\mu}, \frac{2(3+2\epsilon)^2\hat{L}^2 + \mu L'}{2\mu \lambda^2(1-\epsilon)^2} \right\}.$$
\end{proof}

\begin{lemma}\label{lem:bound_C}
Let Assumption \ref{assump:Lip} hold. For any $x \in \St(d, r)^\epsilon$, we have $\normfro{\nabla \cL(x)} \leq C\normfro{\Lambda(x)}$, where $C = 3L'\lambda^{-1}(1-\epsilon)^{-1} +2$.
\end{lemma}
\begin{proof}
We denote the projection of $x$ onto the Stiefel manifold as $x' = \Proj_{\St}(x)$. Given Lemma \ref{lemma:aux}, we know that for any point $x'$ on the manifold, we have $\normfro{\nabla \cL(x')} \leq 2\normfro{\Lambda(x')}.$ Hence,
\begin{equation*}
\resizebox{0.99\hsize}{!}{$
\begin{aligned}
    \normfro{\nabla \cL(x)} 
    &=  \normfro{\nabla \cL(x) - \nabla \cL(x') + \nabla \cL(x')} \\
    &\leq  \normfro{\nabla \cL(x) - \nabla \cL(x')} +  2\normfro{ \Lambda(x')  } \\
    &\leq  \normfro{\nabla \cL(x) - \nabla \cL(x')} +  2\normfro{ \Lambda(x') - \Lambda(x) + \Lambda(x)  } \\
    &\leq  L' \normfro{x - x'} +  2L_\Lambda \normfro{ x - x' }  +  2\normfro{\Lambda(x)} .
  \end{aligned}
  $}
\end{equation*}

Since by definition, $\Lambda(x) = \grad f(x) + \lambda x(x^\top x - I_r)$, we have $\normfro{\Lambda(x) } \geq \normfro{\lambda x(x^\top x - I_r)} \geq \lambda (1-\epsilon) \normfro{x^\top x - I_r}.$ Then, using Lemma \ref{lemma:proj_dist}, we get
\begin{align*}
    \normfro{\nabla \cL(x)} 
    &\leq  3L' \normfro{x - x'}  +  2\normfro{\Lambda(x)} \\
    &\leq 3L' \normfro{x^\top x-I_r}+  2\normfro{\Lambda(x)} \\
    &\leq (3L'\lambda^{-1}(1-\epsilon)^{-1} +2)  \normfro{\Lambda(x)} .
\end{align*}
\end{proof}

\begin{lemma}\label{lem:descent_lemma}
Let Assumption \ref{assump:Lip} hold. For any $x \in \St(d,r)^\epsilon \cap \mathcal{D}(\mathcal{S}, \delta) $ and $\delta \leq 1$, we have the following relationship,
    $$\normfro{\Lambda(x)}^2 
        \leq  \frac{12L'}{\rho^2}(\cL(x) - \cL_\mathcal{S}^\ast ).$$
\end{lemma}

\begin{proof}
Consider $z = x - \alpha \Lambda(x)$ where $\alpha = \frac{\rho}{6L'}$ and recall that $\Lambda(x)$ is Lipschitz continuous with a factor $L'$, implying $$\normfro{\Lambda(x)}=\normfro{\Lambda(x)-\Lambda(x^\ast )}\leq L'\normfro{x-x^\ast}\leq L'\delta \leq L',$$ 
for a local minimizer $x^\ast \in \mathcal{S}$. We then have 
\begin{equation*}
    \begin{aligned}
        \normfro{z^\top z - I_r} =& \normfro{(x - \alpha \Lambda(x))^\top (x - \alpha \Lambda(x)) - I_r}\\
        \leq&\normfro{x^\top x - I_r} + 2\alpha\normfro{ x^\top \Lambda(x) } + \alpha^2 \normfro{\Lambda(x)}^2\\
        \leq & \epsilon + 2\alpha (1+\epsilon) \normfro{\Lambda(x)}+ \alpha^2 \normfro{\Lambda(x)}^2\\
        \leq & \epsilon + 2\alpha (1+\epsilon) L'+ \alpha^2 L'^2.
    \end{aligned}
\end{equation*}
Since $\alpha \leq \frac{1}{L'},$ we have $\normfro{z^\top z - I_r} \leq 3 + 3\epsilon$, and $z\in \St(d, r)^{3+3\epsilon}$. Following the analysis of Proposition 8 in \cite{ablin2023infeasible}, it is easy to show that the function $\cL(z)$ is Lipschitz smooth with a factor 
$L+ (11+9\epsilon)\gamma \leq 6 L'$ when $z\in \St(d, r)^{3+3\epsilon}$
(see Eq. 88 of \cite{ablin2023infeasible}). Using the second condition in Proposition \ref{prop_L} with $\alpha = \frac{\rho}{6L'}$, we get
\begin{align*}
    \cL(z) 
    \leq & \cL(x) - \alpha \rho \normfro{\Lambda(x)}^2 + 3\alpha^2 L' \normfro{\Lambda(x)}^2\\
    \leq & \cL(x) - \frac{\rho^2}{12 L'} \normfro{\Lambda(x)}^2.
\end{align*}
Since $z \in \mathcal{D}(\mathcal{S}, 2\delta) $, we have $\cL_\mathcal{S}^\ast \leq \cL(z)$; therefore,
\begin{align*}
        \normfro{\Lambda(x)}^2 
        \leq &  \frac{12L'}{\rho^2}(\cL(x) - \cL_\mathcal{S}^\ast ).
    \end{align*}
\end{proof}

\subsection{Proofs in Section \ref{sec:linear_system}}

\noindent
\textbf{Proof of Proposition \ref{prop:safe_dist}:} 
We start the proof by observing that for any set of matrices $\mathbf{Z}^\top=[Z_1^\top,\ldots, Z_n^\top]$ and $\mathbf{\bar{Z}}^\top=[\bar{Z}^\top,\ldots, \bar{Z}^\top]$, we have
\begin{equation}
    \label{eq:z_bound}
    \normfro{\mathbf{Z}-\mathbf{\bar{Z}}}\leq \normtwo{I_n-\one\one^\top/n}\normfro{\mathbf{Z}} \leq \normfro{\mathbf{Z}}.
\end{equation}
Next, we show that given the assumption on $\bar{x}_k$ and $\normfro{\X_k - \bar{\X}_k}$, we get $x_{i,k} \in \St(d,r)^{\epsilon}$ since
\begin{equation}\label{eq:conbine_neighborhood}
\resizebox{0.99\hsize}{!}{$
    \begin{aligned}
        &\normfro{x_{i,k}^\top x_{i,k} - I_r} =  \vphantom{\frac{\epsilon}{2}}\normfro{(x_{i,k} - \bar{x}_k + \bar{x}_k)^\top (x_{i,k} - \bar{x}_k + \bar{x}_k) - I_r}\\
        &~~~~~~~~\leq  \vphantom{\frac{\epsilon}{2}} \normfro{\bar{x}_k^\top \bar{x}_k - I_r }+2\normfro{ \bar{x}_k^\top (x_{i,k} - \bar{x}_k )} + \normfro{x_{i,k} - \bar{x}_k}^2\\
        &~~~~~~~~\leq \frac{\epsilon}{2}+2(1+\epsilon)\frac{\epsilon}{10} + (\frac{\epsilon}{10})^2 \leq \epsilon.
    \end{aligned}
    $}
\end{equation}

Applying \eqref{eq:z_bound} to $\mathbf{Z}=\mathbf{\Lambda}_{k} - \mathbf{\Lambda}_{k-1}$ in updates of $\Y$, we get
    \begin{align*}
        \normfro{\Y_{k} - \Bar{\Y}_{k}} 
        \leq &\normfro{\W \Y_{k-1} - \Bar{\Y}_{k-1} }+  \normfro{\mathbf{\Lambda}_{k} - \mathbf{\Lambda}_{k-1}}\\
        \leq & \sigma_W \normfro{\Y_{k-1} - \Bar{\Y}_{k-1} }+  \normfro{\mathbf{\Lambda}_k} +  \normfro{\mathbf{\Lambda}_{k-1}}\\
        \leq & \sigma_W \normfro{\Y_{k-1} - \Bar{\Y}_{k-1}}+ 2\sqrt{n} (G + \lambda \epsilon (1+\epsilon)).
    \end{align*}
    The last inequality is derived from 
    $x_{i,k} \in \St(d,r)^{\epsilon}$, and that $\normfro{\grad f_i(x_{i,k})} \leq G$. Therefore
    $$
    \normfro{\Y_{k} - \Bar{\Y}_k }\leq \frac{2\sqrt{n}(G+\lambda \epsilon (1+\epsilon))}{1-\sigma_W}.
    $$
    We also know that $\X_{k+1} = \W \X_{k} - \alpha \Y_k$; then, $\X_{k+1} - \Bar{\X}_{k+1} = \W (\X_{k} - \Bar{\X}_k) - \alpha (\Y_k - \Bar{\Y}_k)$, and we get,
    \begin{align*}
        \normfro{\X_{k+1} - \Bar{\X}_{k+1}} \vphantom{\frac{\sqrt{n}(G )}{1-\sigma_W}}
        = & \normfro{\W (\X_{k} - \Bar{\X}_k) - \alpha (\Y_k - \Bar{\Y}_k)}\\
        \leq &\normfro{\W (\X_{k} - \Bar{\X}_k) }+ \alpha \normfro{(\Y_k - \Bar{\Y}_k)}\vphantom{\frac{\sqrt{n}(G )}{1-\sigma_W}}\\
        \leq & \sigma_W \normfro{\X_{k} - \Bar{\X}_k }+ \alpha \normfro{\Y_k - \Bar{\Y}_k}\vphantom{\frac{\sqrt{n}(G )}{1-\sigma_W}}\\
        \leq & \sigma_W \normfro{\X_{k} - \Bar{\X}_k }+ \frac{2 \alpha\sqrt{n}(G+\lambda \epsilon (1+\epsilon))}{1-\sigma_W}.
    \end{align*}
    If we choose
    \begin{equation}
        \alpha \leq \frac{ (1-\sigma_W)^2\epsilon }{20 \sqrt{n}(G+\lambda \epsilon (1+\epsilon))},
    \end{equation}
    we have $\normfro{\X_{k+1} - \Bar{\X}_{k+1}} \leq \frac{2\sqrt{n}\alpha(G+\lambda \epsilon (1+\epsilon))}{(1-\sigma_W)^2} \leq \frac{\epsilon}{10}$.

    Next, we bound the distance of $\bar{x}_{k+1}$ from $\St(d,r)$. Since $\Lambda(\Bar{x}_k)=\grad f(\Bar{x}_k) + \lambda \nabla p(\Bar{x}_k)$, from the algorithm update we have 
\begin{align*}
    \Bar{x}_{k+1} = & 
    \Bar{x}_k - \alpha (\grad f(\Bar{x}_k) + \lambda \nabla p(\Bar{x}_k) +   \Bar{y}_k - \Lambda(\Bar{x}_k) ).
\end{align*}
For the ease of notation, we denote $\Delta_k \triangleq \Bar{x}_k^\top \Bar{x}_k - I_r$ and $D_k \triangleq  \Bar{y}_k - \Lambda(\Bar{x}_k) $, and we then have
\begin{equation*}
\resizebox{0.99\hsize}{!}{$
    \begin{aligned}
        \Delta_{k+1} = &\Bar{x}_{k+1}^\top \Bar{x}_{k+1} - I_r\\
        =& (\Bar{x}_k - \alpha (\Lambda(\Bar{x}_k)+ D_k ))^\top (\Bar{x}_k - \alpha (\Lambda(\Bar{x}_k)+ D_k )) - I_r\\
        =& \Bar{x}_{k}^\top \Bar{x}_{k} - I_r + \alpha^2 (\Lambda(\Bar{x}_k) + D_k)^\top (\Lambda(\Bar{x}_k) + D_k) \\
        &- \alpha (\Bar{x}_k^\top (\lambda \nabla p(\Bar{x}_k) + D_k ) + (\lambda \nabla p(\Bar{x}_k) + D_k )^\top \Bar{x}_k )\\
        & - \alpha (\Bar{x}_{k}^\top \grad f(\Bar{x}_k) + \grad f(\Bar{x}_k)^\top \Bar{x}_{k}).
    \end{aligned}
    $}
\end{equation*}
For the last line, given the definition of $\grad f(x)$,
$$\Bar{x}_{k}^\top \grad f(\Bar{x}_k) + \grad f(\Bar{x}_k)^\top \Bar{x}_{k} = 0.$$
Since $\nabla p(\Bar{x}_k) = \Bar{x}_k\Delta_k$ we can write
\begin{equation*}
    \begin{aligned}
        \Delta_{k+1} 
        =& \Delta_k - \alpha ( 2 \lambda\Delta_k(I+\Delta_k) +\Bar{x}_k^\top  D_k  +  D_k^\top \Bar{x}_k )\\
        &+ \alpha^2 (\Lambda(\Bar{x}_k) + D_k)^\top (\Lambda(\Bar{x}_k) + D_k),
    \end{aligned}
\end{equation*}
which implies
\begin{equation*}
    \begin{aligned}
        \normfro{\Delta_{k+1}} \leq& (1- 2\alpha \lambda )\normfro{\Delta_{k}} + 2 \alpha \lambda \normfro{\Delta_{k}}^2+ 4 \alpha\normfro{D_k} \\
        &+ 2\alpha^2 (\normfro{\Lambda(\Bar{x}_k)}^2 + \normfro{D_k}^2).
    \end{aligned}
\end{equation*}
Given $\alpha \leq \frac{1}{2\lambda}, 1-2\alpha \lambda \geq 0$, and $\normfro{\Lambda(\Bar{x}_k)}^2 \leq G^2 + \lambda^2(1+\epsilon)\epsilon^2$ for $\Bar{x}_k \in \St(d, r)^{\frac{\epsilon}{2}}$, we have
\begin{equation*}
    \begin{aligned}
        \normfro{\Delta_{k+1}} &\leq (1- 2\alpha \lambda )\frac{\epsilon}{2} + 2 \alpha \lambda \frac{\epsilon^2}{4}+ 4 \alpha\normfro{D_k}\\
        &+ 2\alpha^2 (G^2 + \lambda^2(1+\epsilon)\epsilon^2 + \normfro{D_k}^2).  
    \end{aligned}
\end{equation*}
Observe that due to the initialization 
\begin{align}\label{eq:lipcont}
\normfro{\bar{y}_{k}-\Lambda(\Bar{x}_k)}&=\normfro{\frac{1}{n}\sum_{i=1}^n\Lambda_i(x_{i,k})-\Lambda(\Bar{x}_k)}\nonumber\\
&=\normfro{\frac{1}{n}\sum_{i=1}^n\Lambda_i(x_{i,k})-\frac{1}{n}\sum_{i=1}^n\Lambda_i(\Bar{x}_k)}\nonumber\\
&\leq \frac{L'}{n}\sum_{i=1}^n \normfro{x_{i,k}-\Bar{x}_k}\nonumber\\
&\leq \frac{L'}{\sqrt{n}}\normfro{\X_{k} -\Bar{\X}_k}.
\end{align}
We need $\normfro{\Delta_{k+1}}\leq \frac{\epsilon}{2}$ and we know from above that
\begin{align*}
      \normfro{D_k} &= \normfro{ \Bar{y}_k - \Lambda(\Bar{x}_k) }\leq \frac{L'}{\sqrt{n}}\frac{2 \alpha\sqrt{n}(G+\lambda \epsilon (1+\epsilon))}{(1-\sigma_W)^2}.
\end{align*}
If we set $\alpha \leq  \frac{\lambda \epsilon^2 (1 - \sigma_W)^2}{16 L'(G+\lambda \epsilon (1+\epsilon))}$ such that $\normfro{D_k} \leq \frac{\lambda\epsilon^2}{8}$,
a sufficient condition for $\alpha$ such that $\normfro{\Delta_{k+1}}\leq \frac{\epsilon}{2}$ is obtained as follows
\begin{equation*}
    \begin{aligned}
        2\alpha (G^2 + \lambda^2 &(1+\epsilon)\epsilon^2 + \frac{\epsilon^4 \lambda^2}{16}  )  -  \lambda {\epsilon} +    \lambda\epsilon^2\leq 0\\
        \Rightarrow\alpha \leq &\frac{\lambda \epsilon (1-\epsilon)}{2(G^2 + \lambda^2 (1+\epsilon)\epsilon^2 + \frac{\epsilon^4 \lambda^2}{16})}.
    \end{aligned}
\end{equation*}
Combining all the requirements above on $\alpha$, we get
\begin{equation*}
\begin{aligned}
    \alpha_{safe} =& \min\bigg\{
    \frac{ (1-\sigma_W)^2\epsilon }{20 \sqrt{n}(G+\lambda \epsilon (1+\epsilon))}, \frac{\lambda \epsilon^2 (1 - \sigma_W)^2}{16 L'(G+\lambda \epsilon (1+\epsilon))},\\
    &\frac{1}{2\lambda}, \frac{\lambda \epsilon (1-\epsilon)}{2(G^2 + \lambda^2 (1+\epsilon)\epsilon^2 + \frac{\epsilon^4 \lambda^2}{16})}
    \bigg\}.
\end{aligned}
\end{equation*}


\noindent
\textbf{Proof of Lemma \ref{lem:bound_x}:}
Since $\sigma_W$ is the second largest singular value of $W$, from Assumption \ref{assump_net} the following holds for all $\kappa >0$ by AM-GM inequality
\begin{align*}
\normfro{\X_k - \Bar{\X}_k}^2  
    &= \vphantom{\frac{1 + \kappa}{\kappa}\alpha^2}\normfro{\W  (\X_{k-1}- \Bar{\X}_{k-1} ) - \alpha (\Y_{k-1} - \Bar{\Y}_{k-1}) }^2\\
    &\vphantom{\frac{1 + \kappa}{\kappa}\alpha^2}\leq (1+\kappa)\normfro{\W  (\X_{k-1}- \Bar{\X}_{k-1} )}^2 \\
    &\quad+ \frac{1 + \kappa}{\kappa}\alpha^2 \normfro{\Y_{k-1} - \Bar{\Y}_{k-1}}^2 \\
    &\leq (1+\kappa) \sigma_W^2\normfro{ \X_{k-1}- \Bar{\X}_{k-1} }^2 \\
    &\quad+ \frac{1 + \kappa}{\kappa}\alpha^2 \normfro{\Y_{k-1} - \Bar{\Y}_{k-1}}^2.
\end{align*}
We let $\kappa = \frac{1-\sigma_W^2}{2 \sigma_W^2}$ and get:
\begin{align*}
    \normfro{\X_k - \Bar{\X}_k}^2\leq &\frac{1+\sigma_W^2}{2} \normfro{ \X_{k-1}- \Bar{\X}_{k-1} }^2\\
    &+ \frac{1 + \sigma_W^2}{1-\sigma_W^2}\alpha^2 \normfro{\Y_{k-1} - \Bar{\Y}_{k-1}}^2,
\end{align*}
which completes the proof.

\qed

\noindent
\textbf{Proof of Lemma \ref{lem:bound_y}:}  
By definition of the update, we have
\begin{align*}
    &\normfro{\Y_k - \Bar{\Y}_k}^2 \\
    = &
    \normfro{\W(\Y_{k-1} - \Bar{\Y}_{k-1}) + (\mathbf{\Lambda}_{k}- \Bar{\mathbf{\Lambda}}_k - \mathbf{\Lambda}_{k-1} + \Bar{\mathbf{\Lambda}}_{k-1})}^2.
    \end{align*}
    Using the same technique as Lemma \ref{lem:bound_x}, we get
\begin{equation*}
\resizebox{0.99\hsize}{!}{$
\begin{aligned}
& \frac{1 + \sigma_W^2}{1-\sigma_W^2} \normfro{\mathbf{\Lambda}_{k}- \Bar{\mathbf{\Lambda}}_k - \mathbf{\Lambda}_{k-1} + \Bar{\mathbf{\Lambda}}_{k-1}}^2\\
    \leq &  \frac{1 + \sigma_W^2}{1-\sigma_W^2} \normfro{\mathbf{\Lambda}_{k} - \mathbf{\Lambda}_{k-1} }^2\\
    \leq &   \frac{1 + \sigma_W^2}{1-\sigma_W^2}L'^2 \normfro{\X_{k} - \X_{k-1} }^2\\
    = &  \frac{1 + \sigma_W^2}{1-\sigma_W^2}L'^2 \normfro{ (\W - \bI)\X_{k-1}- \alpha \Y_{k-1}  }^2\\
    \leq & \frac{1 + \sigma_W^2}{1-\sigma_W^2}L'^2 (2\normfro{(\W - \bI)\X_{k-1} }^2 + 2\normfro{\alpha \Y_{k-1}  }^2)\\
    \leq & \frac{1 + \sigma_W^2}{1-\sigma_W^2}L'^2 (2\normfro{(\W - \bI)(\X_{k-1}-\bar{\X}_{k-1}) }^2 + 2\normfro{\alpha \Y_{k-1}  }^2)\\
    \leq &  8L'^2 \frac{1 + \sigma_W^2}{1-\sigma_W^2}\normfro{\X_{k-1} - \Bar{\X}_{k-1} }^2 + 2L'^2\alpha^2\frac{1 + \sigma_W^2}{1-\sigma_W^2}\normfro{ \Y_{k-1}  }^2.
  \end{aligned}
  $}
\end{equation*}
Therefore, we have
\begin{equation}\label{eq:y_consensus_bound_w_yk1}
\resizebox{0.99\hsize}{!}{$
\begin{aligned}
\normfro{\Y_k -& \Bar{\Y}_k}^2 
    \leq 8L'^2 \frac{1 + \sigma_W^2}{1-\sigma_W^2}\normfro{\X_{k-1} - \Bar{\X}_{k-1} }^2  \\
    &+\frac{1+\sigma_W^2}{2} \normfro{\Y_{k-1} - \Bar{\Y}_{k-1}}^2
    + 2L'^2\alpha^2\frac{1 + \sigma_W^2}{1-\sigma_W^2}\normfro{ \Y_{k-1}  }^2.
  \end{aligned}
  $}
\end{equation}

Now, for $\normfro{ \Y_{k-1}  }^2$, we provide the following decomposition:
\begin{align*}
        \normfro{ \Y_{k-1}  }^2 = &\normfro{ \Y_{k-1} -  \Bar{\Y}_{k-1} + \Bar{\Y}_{k-1} }^2\\
        \leq & 2 \normfro{ \Y_{k-1} -   \Bar{\Y}_{k-1} }^2 + 2\normfro{ \Bar{\Y}_{k-1} }^2.
\end{align*}
Plugging the above inequality into \eqref{eq:y_consensus_bound_w_yk1}, we get
\begin{equation*}
\resizebox{0.99\hsize}{!}{$
\begin{aligned}
    &\normfro{\Y_k - \Bar{\Y}_k}^2 \vphantom{\left(\frac{1+\sigma_W^2}{2} + 4L'^2\alpha^2\frac{1 + \sigma_W^2}{1-\sigma_W^2}\right)}\\
    \leq & \left(\frac{1+\sigma_W^2}{2} + 4L'^2\alpha^2\frac{1 + \sigma_W^2}{1-\sigma_W^2}\right) \normfro{\Y_{k-1} - \Bar{\Y}_{k-1}}^2 \\
    &+8L'^2 \frac{1 + \sigma_W^2}{1-\sigma_W^2}\normfro{\X_{k-1} - \Bar{\X}_{k-1} }^2 + 4L'^2\alpha^2\frac{1 + \sigma_W^2}{1-\sigma_W^2}\normfro{  \Bar{\Y}_{k-1}  }^2 \vphantom{\left(\frac{1+\sigma_W^2}{2} + 4L'^2\alpha^2\frac{1 + \sigma_W^2}{1-\sigma_W^2}\right)}.
  \end{aligned}
  $}
\end{equation*}

\qed

\noindent
\textbf{Proof of Theorem \ref{thm:stable}:} Combining the results of Lemma \ref{lem:bound_x} and Lemma \ref{lem:bound_y}, we can easily verify the linear system relationship. Then, we need to ensure that $\rho(\Tilde{G})<1$. Based on Lemma \ref{lem:spectral_rad}, $\rho(\Tilde{G})<1$ holds if there exists a solution $s_1,s_2>0$ for the following inequalities
\begin{align*}
    \left(\frac{1+\sigma_W^2}{2} + 4L'^2\alpha^2\frac{1 + \sigma_W^2}{1-\sigma_W^2}\right)s_1 + 8 \frac{1 + \sigma_W^2}{1-\sigma_W^2}s_2 &< s_1,\\
    \frac{1 + \sigma_W^2}{1-\sigma_W^2}\alpha^2 L'^2s_1 + \frac{1+\sigma_W^2}{2}s_2 &< s_2.
\end{align*}
Simplifying these equations, we get 
\begin{align*}
        2\frac{1 + \sigma_W^2}{(1-\sigma_W^2)^2}\alpha^2 L'^2s_1 &< s_2,\\
        8\frac{1 + \sigma_W^2}{1-\sigma_W^2}s_2 &< \left(\frac{1-\sigma_W^2}{2} - 4L'^2\alpha^2\frac{1 + \sigma_W^2}{1-\sigma_W^2}\right)s_1.
\end{align*}
After removing $s_2$, we get a sufficient condition for $\alpha$ such that 
$$\alpha < \frac{(1-\sigma_W^2)^2}{1+\sigma_W^2} \frac{1}{16L'}.$$

\qed

\noindent
\textbf{Proof of Lemma \ref{lem:decent}:} We first start with the Lipschitz smoothness of $\mathcal{L}$ and the updates of the algorithm to get
\begin{equation*}
\begin{aligned}
&\cL(\Bar{x}_k) -  \cL(\Bar{x}_{k-1}) \vphantom{\frac{L' }{2}}\\
\leq &   \langle \nabla \cL (\Bar{x}_{k-1}), \Bar{x}_{k} - \Bar{x}_{k-1}\rangle + \frac{L' }{2}\normfro{\Bar{x}_{k} - \Bar{x}_{k-1}}^2\\
= &   - \alpha\langle \nabla \cL (\Bar{x}_{k-1}), \bar{y}_{k-1}\rangle + \frac{\alpha^2 L' }{2}\normfro{\bar{y}_{k-1}}^2\\
= &    - \alpha\langle \nabla \cL (\Bar{x}_{k-1}), \bar{y}_{k-1} - \Lambda(\Bar{x}_{k-1}) + \Lambda(\Bar{x}_{k-1})\rangle \vphantom{\frac{L' }{2}}\\
&+ \frac{\alpha^2 L' }{2}\normfro{\bar{y}_{k-1}}^2\\
\leq & \vphantom{\frac{L' }{2}}  - \alpha \rho \normfro{\Lambda(\Bar{x}_{k-1})}^2 - \alpha\langle \nabla \cL (\Bar{x}_{k-1}), \bar{y}_{k-1} - \Lambda(\Bar{x}_{k-1}) \rangle \\
&+ \frac{\alpha^2 L' }{2}\normfro{\bar{y}_{k-1}}^2\\
\leq & \vphantom{\frac{L' }{2}}  - \alpha \rho \normfro{\Lambda(\Bar{x}_{k-1})}^2 + \frac{\alpha}{2 \kappa } \normfro{\nabla \cL (\Bar{x}_{k-1})}^2 \\
&+ \frac{\alpha \kappa}{2} \normfro{\bar{y}_{k-1} - \Lambda(\Bar{x}_{k-1})}^2 + \frac{\alpha^2 L' }{2}\normfro{\bar{y}_{k-1}}^2,
  \end{aligned}
\end{equation*}
 where the second inequality used the results from Proposition \ref{prop_L} and the last inequality holds for any $\kappa>0$ due to AM-GM inequality. Now, with the help of Lemma \ref{lem:bound_C} and setting $\kappa = \frac{C^2}{\rho }$ in above, we get
\begin{equation*}
\resizebox{0.99\hsize}{!}{$
\begin{aligned}
&\cL(\Bar{x}_k) -  \cL(\Bar{x}_{k-1}) \vphantom{\frac{\alpha C^2}{2 \rho}} \\
\leq &   - \alpha \rho \normfro{\Lambda(\Bar{x}_{k-1})}^2 + \frac{\alpha \rho}{2 C^2 } \normfro{\nabla \cL (\Bar{x}_{k-1})}^2 \\
&+ \frac{\alpha C^2}{2 \rho} \normfro{\bar{y}_{k-1} - \Lambda(\Bar{x}_{k-1})}^2 + \frac{\alpha^2 L' }{2}\normfro{\bar{y}_{k-1}}^2\\
\leq &   - \frac{\alpha \rho}{2 } \normfro{\Lambda(\Bar{x}_{k-1})}^2  + \frac{\alpha C^2}{2 \rho} \normfro{\bar{y}_{k-1} - \Lambda(\Bar{x}_{k-1})}^2  + \frac{\alpha^2 L' }{2}\normfro{\bar{y}_{k-1}}^2\\
\leq &   - \frac{\alpha \rho}{2 } \normfro{\Lambda(\Bar{x}_{k-1})}^2  + \frac{\alpha C^2 L'^2}{2 \rho n} \normfro{\bar{\X}_{k-1} - \X_{k-1}}^2 + \frac{\alpha^2 L' }{2}\normfro{\bar{y}_{k-1}}^2,
  \end{aligned}
  $}
\end{equation*}
where the last inequality is due to \eqref{eq:lipcont}. The proof is completed.
\qed

\subsection{Proofs in Section \ref{sec:conv}}

\noindent
\textbf{Proof of Corollary \ref{cor:sum_x}:} We first calculate the determinant of 
$I - \Tilde{G}$. Given that
\begin{align*}
    I - \Tilde{G} & =\begin{bmatrix}
    1-\frac{1+\sigma_W^2}{2} - 4L'^2\alpha^2\frac{1 + \sigma_W^2}{1-\sigma_W^2} & -8\frac{1 + \sigma_W^2}{1-\sigma_W^2}\\
        -\frac{1 + \sigma_W^2}{1-\sigma_W^2}\alpha^2 L'^2&1-\frac{1+\sigma_W^2}{2}
    \end{bmatrix},
\end{align*}
we can explicitly write
\begin{equation*}
\resizebox{0.99\hsize}{!}{$
\begin{aligned}
det(I - \Tilde{G}) 
     = &(\frac{1-\sigma_W^2}{2})^2 -2\alpha^2 L'^2 \Big((1 + \sigma_W^2)+4 \frac{ (1+\sigma_W^2)^2}{(1-\sigma_W^2)^2}\Big).
  \end{aligned}
  $}
\end{equation*}
Then, in order to have $det(I - \Tilde{G}) \geq \frac{(1-\sigma_W^2)^2}{8}$, we need
\begin{align*}
    2\alpha^2 L'^2 \Big((1 + \sigma_W^2)+4 \frac{ (1+\sigma_W^2)^2}{(1-\sigma_W^2)^2}\Big) & \leq \frac{(1-\sigma_W^2)^2}{8},
\end{align*}
which is satisfied when $\alpha < \frac{(1-\sigma_W^2)^2}{1+\sigma_W^2} \frac{1}{16L'}$. Next, we note that
\begin{align*}
    &(I - \Tilde{G} )^{-1} = \frac{(I - \Tilde{G} )^*}{det (I - \Tilde{G} )}  \\
     \leq &\frac{8}{(1-\sigma_W^2)^2} \begin{bmatrix}
    \frac{1-\sigma_W^2}{2} - 4L'^2\alpha^2\frac{1 + \sigma_W^2}{1-\sigma_W^2} & -8\frac{1 + \sigma_W^2}{1-\sigma_W^2}\\
        -\frac{1 + \sigma_W^2}{1-\sigma_W^2}\alpha^2 L'^2&\frac{1-\sigma_W^2}{2}
    \end{bmatrix}^*\\
      = &\frac{8}{(1-\sigma_W^2)^2} \begin{bmatrix}
     \frac{1-\sigma_W^2}{2}& 8\frac{1 + \sigma_W^2}{1-\sigma_W^2}\\
        \frac{1 + \sigma_W^2}{1-\sigma_W^2}\alpha^2 L'^2&\frac{1-\sigma_W^2}{2} - 4L'^2\alpha^2\frac{1 + \sigma_W^2}{1-\sigma_W^2}
    \end{bmatrix}\\
     =  &\begin{bmatrix}
     \frac{4}{1-\sigma_W^2}& 64\frac{1 + \sigma_W^2}{(1-\sigma_W^2)^3}\\
        \frac{1 + \sigma_W^2}{(1-\sigma_W^2)^3} 8\alpha^2 L'^2&\frac{4}{1-\sigma_W^2} - 32 L'^2\alpha^2\frac{1 + \sigma_W^2}{(1-\sigma_W^2)^3}
    \end{bmatrix}.
\end{align*}
The above inequality is element-wise, with each entry non-negative.

Now, given that $\Tilde{\xi}_k \leq \Tilde{G}\Tilde{\xi}_{k-1} + \Tilde{u}_{k-1}$ and that $\Tilde{\xi}_0 = 0$, we have
\begin{align*}
    \Tilde{\xi}_k \leq   \sum_{t=0}^{k-1} \Tilde{G}^t\Tilde{u}_{k-t-1}
   \Rightarrow \sum_{k=1}^K\Tilde{\xi}_k \leq (I - \Tilde{G} )^{-1}  \sum_{k=0}^{K-1} \Tilde{u}_{k} .
\end{align*}
Given the inequality above, we know that
\begin{equation*}
\resizebox{0.99\hsize}{!}{$
\begin{aligned}
\sum_k L' &\normfro{\X_k -  \Bar{\X}_k}^2 \leq \frac{1 + \sigma_W^2}{(1-\sigma_W^2)^3} 8\alpha^2 L'^2  \sum_k 4L'\alpha^2\frac{1 + \sigma_W^2}{1-\sigma_W^2}\normfro{  \Bar{\Y}_{k-1}  }^2.
  \end{aligned}
  $}
\end{equation*}
Therefore, we can write
\begin{align*}
    \sum_k  \normfro{\X_k -  \Bar{\X}_k}^2 &\leq  \frac{(1 + \sigma_W^2)^2}{(1-\sigma_W^2)^4} 32\alpha^4 L'^2  \sum_k\normfro{  \Bar{\Y}_{k-1}  }^2,
\end{align*}
which completes the proof. 

\qed

\noindent
\textbf{Proof of Theorem \ref{thm:ergodic}:} We first state that $\alpha  < \frac{(1-\sigma_W^2)^2}{1+\sigma_W^2} \frac{1}{16L'}$ guarantees the stability of the system \eqref{eq:system} and enables the use of Corollary \ref{cor:sum_x}. We have that
\begin{align}\label{eq:ybarminuslambda}
    \normfro{\bar{y}_{k-1}}^2 = & \normfro{\Lambda(\Bar{x}_{k-1}) + (\bar{y}_{k-1} - \Lambda(\Bar{x}_{k-1}))}^2\nonumber\\
    \leq & 2 \normfro{\Lambda(\Bar{x}_{k-1})}^2 + 2 \normfro{\bar{y}_{k-1} - \Lambda(\Bar{x}_{k-1})}^2.
\end{align}
Combining this with Lemma \ref{lem:decent}, we get
\begin{equation*}
\resizebox{0.99\hsize}{!}{$
\begin{aligned}
\vphantom{\Big(\frac{\alpha C^2 L'^2}{2 \rho n} + \frac{2\alpha^2 L'^3 }{n} \Big) }&\cL(\Bar{x}_k) -  \cL(\Bar{x}_{k-1})\\
\vphantom{\Big(\frac{\alpha C^2 L'^2}{2 \rho n} + \frac{2\alpha^2 L'^3 }{n} \Big) }\leq &   - \frac{\alpha \rho}{2 } \normfro{\Lambda(\Bar{x}_{k-1})}^2  + \frac{\alpha C^2 L'^2}{2 \rho n} \normfro{\bar{\X}_{k-1} - \X_{k-1}}^2 + \frac{\alpha^2 L' }{2}\normfro{\bar{y}_{k-1}}^2\\
\vphantom{\Big(\frac{\alpha C^2 L'^2}{2 \rho n} + \frac{2\alpha^2 L'^3 }{n} \Big) }\leq &  - \frac{\alpha \rho}{2 } \normfro{\Lambda(\Bar{x}_{k-1})}^2  + \frac{\alpha C^2 L'^2}{2 \rho n} \normfro{\bar{\X}_{k-1} - \X_{k-1}}^2 - \frac{\alpha^2 L' }{2}\normfro{\bar{y}_{k-1}}^2\\
\vphantom{\Big(\frac{\alpha C^2 L'^2}{2 \rho n} + \frac{2\alpha^2 L'^3 }{n} \Big) }& + 2\alpha^2 L' \normfro{\Lambda(\Bar{x}_{k-1})}^2 + 2\alpha^2 L' \normfro{\bar{y}_{k-1} - \Lambda(\Bar{x}_{k-1})}^2\\
\vphantom{\Big(\frac{\alpha C^2 L'^2}{2 \rho n} + \frac{2\alpha^2 L'^3 }{n} \Big) }\leq &  - \left( \frac{\alpha \rho}{2 } - 2\alpha^2 L'\right) \normfro{\Lambda(\Bar{x}_{k-1})}^2  \\
\vphantom{\Big(\frac{\alpha C^2 L'^2}{2 \rho n} + \frac{2\alpha^2 L'^3 }{n} \Big) }&+ \Big(\frac{\alpha C^2 L'^2}{2 \rho n} + \frac{2\alpha^2 L'^3 }{n} \Big) \normfro{\bar{\X}_{k-1} - \X_{k-1}}^2 - \frac{\alpha^2 L' }{2}\normfro{\bar{y}_{k-1}}^2 .
  \end{aligned}
  $}
\end{equation*}
Note that $\alpha \leq \frac{\rho}{8L'}$, and in addition we know from Proposition \ref{prop_L} that $ \rho \leq \frac{1}{2}$. Therefore,
  \begin{align*}
&\cL(\Bar{x}_k) -  \cL(\Bar{x}_{k-1}) \leq -  \frac{\alpha \rho}{4 } \normfro{\Lambda(\Bar{x}_{k-1})}^2  \\
& 
+ \frac{\alpha  L'^2C^2}{\rho n}  \normfro{\bar{\X}_{k-1} - \X_{k-1}}^2 - \frac{\alpha^2 L' }{2}\normfro{\bar{y}_{k-1}}^2 .
    \end{align*}
Rearranging and summing both sides over $k = 1, ..., K$ and using Corollary \ref{cor:sum_x}, we have
\begin{equation*}
    \begin{aligned}
&\sum_k\frac{\alpha \rho}{4 } \normfro{\Lambda(\Bar{x}_{k-1})}^2 -(\cL(\Bar{x}_{0}) -\cL(\Bar{x}_K))\\
    \leq &   - \frac{\alpha^2 L' }{2}\sum_k \normfro{\bar{y}_{k-1}}^2   +  \frac{\alpha  L'^2C^2}{\rho n}\sum_k \normfro{\bar{\X}_{k-1} - \X_{k-1}}^2 \\
    \leq &   - \frac{\alpha^2 L' }{2}\sum_k \normfro{\bar{y}_{k-1}}^2 \\
    &+ \frac{\alpha  L'^2C^2}{\rho n}\frac{(1 + \sigma_W^2)^2}{(1-\sigma_W^2)^4} 32\alpha^4 L'^2  \sum_k\normfro{  \Bar{\Y}_{k-1}  }^2 \\
    =&  - \Big(\frac{\alpha^2 L' }{2} - \frac{\alpha  L'^2C^2}{\rho }\frac{(1 + \sigma_W^2)^2}{(1-\sigma_W^2)^4} 32\alpha^4 L'^2 \Big) \sum_k\normfro{  \Bar{y}_{k-1}  }^2 .
    \end{aligned}
\end{equation*}
Selecting $\alpha \leq \frac{1}{4L'}\sqrt[3]{\frac{\rho (1-\sigma_W^2)^4}{(1 + \sigma_W^2)^2 C^2}}$ such that $ \frac{\alpha^2 L' }{2} - \frac{\alpha  L'^2C^2}{\rho }\frac{(1 + \sigma_W^2)^2}{(1-\sigma_W^2)^4} 32\alpha^4 L'^2\geq 0$, we get the desired result
\begin{align*}
    &\frac{\sum_k \normfro{\Lambda(\Bar{x}_{k-1})}^2 }{K}
    \leq \frac{1}{K}\frac{4}{\alpha \rho}(\cL(\Bar{x}_{0}) -\cL(\Bar{x}_K)).
\end{align*}

In addition, the consensus error can be bounded with 
\begin{equation*}
    \begin{aligned}
    \sum_k  \normfro{\X_k -  \Bar{\X}_k}^2 
    &\leq  \frac{(1 + \sigma_W^2)^2}{(1-\sigma_W^2)^4} 32\alpha^4 L'^2  \sum_k\normfro{  \Bar{\Y}_{k-1}  }^2\\
    &\leq   \frac{(1 + \sigma_W^2)^2}{(1-\sigma_W^2)^4} 64\alpha^4 L'^2 n  \sum_k\normfro{\Lambda(\Bar{x}_{k-1})}^2 \\
    &+ \frac{(1 + \sigma_W^2)^2}{(1-\sigma_W^2)^4} {64\alpha^4 L'^4} \sum_k  \normfro{\X_{k-1} -  \Bar{\X}_{k-1}}^2,
    \end{aligned}
\end{equation*}
where we used \eqref{eq:lipcont} and \eqref{eq:ybarminuslambda} again. Since we have $$\alpha \leq \frac{1}{4L'}\sqrt[3]{\frac{\rho (1-\sigma_W^2)^4}{(1 + \sigma_W^2)^2 C^2}},$$
and $\rho \leq \frac{1}{2}, C>1$, it is easy to show that $\frac{(1 + \sigma_W^2)^2}{(1-\sigma_W^2)^4} {64\alpha^4 L'^4} \leq \frac{1}{2}$; hence, we can write
\begin{equation*}
\resizebox{0.99\hsize}{!}{$
\begin{aligned}
    \sum_k  \normfro{\X_k -  \Bar{\X}_k}^2 
    \leq&  \frac{(1 + \sigma_W^2)^2}{(1-\sigma_W^2)^4} 128n\alpha^4 L'^2  \sum_k\normfro{\Lambda(\Bar{x}_{k-1})}^2,
  \end{aligned}
  $}
\end{equation*}
which completes the proof.

\qed

\subsection{Proofs in Section \ref{sec:conv_PL}}
\noindent
\textbf{Proof of Lemma
 \ref{lem:KL-QG}:} Since $x \in \St(d, r)^\epsilon$, given Proposition \ref{prop_L}, it is clear that,
    $$\normfro{\nabla \cL(x)} \geq \rho \normfro{ \Lambda(x)}.$$
Combined with Lemma \ref{lem:pseudo-grad-dominate}, we have,
$$\mathcal{L}(x) -\cL_\mathcal{S}^\ast  \leq \frac{1}{\mu'} \normfro{\Lambda(x)}^2 \leq \frac{1}{\mu' \rho^2} \normfro{\nabla \cL(x)}^2.$$
    Therefore, we know that the local Euclidean PŁ condition for $\cL(x)$ holds on $\forall x \in \mathcal{D}(\mathcal{S}, \delta) \cap \St(d, r)^\epsilon$. Given Proposition 2.2 in \cite{rebjock2023fast}, 
the local Euclidean PŁ condition on $\cL(x)$ implies the quadratic growth relationship below,
$$\cL(x) - \cL_\mathcal{S}^\ast\geq \frac{\mu' \rho^2}{4} \text{dist}( \mathcal{S}, x)^2.$$
\qed

\noindent
\textbf{Proof of Lemma \ref{lem:dist_V}:} Without  loss of generality and for the ease of notation, in this proof, we consider the case where $ \cL_{\mathcal{S}}^*=0$. By the smoothness of $\cL$ and the algorithm update, we have
\begin{equation*}
\resizebox{0.99\hsize}{!}{$
    \begin{aligned}
        \cL(\Bar{x}_k)  \leq & \cL(\Bar{x}_{k-1})  + \langle \nabla \cL (\Bar{x}_{k-1}), \Bar{x}_{k} - \Bar{x}_{k-1}\rangle + \frac{L' }{2}\normfro{\Bar{x}_{k} - \Bar{x}_{k-1}}^2\\
        \vphantom{\frac{L' }{2}}= & \cL(\Bar{x}_{k-1})  - \langle \nabla \cL (\Bar{x}_{k-1}), \alpha \Bar{y}_{k-1}\rangle + \frac{L' \alpha^2}{2}\normfro{\Bar{y}_{k-1}}^2\\
        \vphantom{\frac{L' }{2}}= & \cL(\Bar{x}_{k-1})  - \alpha\langle \nabla \cL (\Bar{x}_{k-1}), \Lambda(\Bar{x}_{k-1})\rangle \\
        \vphantom{\frac{L' }{2}}+& \alpha\langle \nabla \cL (\Bar{x}_{k-1}), \Lambda(\Bar{x}_{k-1}) - \Bar{y}_{k-1}\rangle + \frac{L' \alpha^2}{2}\normfro{\Bar{y}_{k-1}}^2.
    \end{aligned}
$}
\end{equation*}
        Applying Proposition \ref{prop_L}, we derive
        \begin{align*}
        \vphantom{\frac{L' }{2}}\cL(\Bar{x}_k)  \leq & \cL(\Bar{x}_{k-1}) -\alpha\rho \normfro{\Lambda(\Bar{x}_{k-1})}^2  \\
       \vphantom{\frac{L' }{2}} +& \alpha\langle \nabla \cL (\Bar{x}_{k-1}), \Lambda(\Bar{x}_{k-1}) - \Bar{y}_{k-1}\rangle + \frac{L'\alpha^2}{2}\normfro{\Bar{y}_{k-1}}^2\\
       \vphantom{\frac{L' }{2}} = & \cL(\Bar{x}_{k-1}) -\alpha\rho \normfro{\Lambda(\Bar{x}_{k-1})}^2  \\
       \vphantom{\frac{L' }{2}} +& \alpha\langle \nabla \cL (\Bar{x}_{k-1}), \Lambda(\Bar{x}_{k-1}) - \Bar{y}_{k-1}\rangle \\
       \vphantom{\frac{L' }{2}} +& \frac{L'\alpha^2}{2}\normfro{\Bar{y}_{k-1} -\Lambda(\Bar{x}_{k-1}) +\Lambda(\Bar{x}_{k-1})}^2.
        \end{align*}
Applying the AM-GM inequality on the last term as well as the inner product term (with a factor $\eta > 0$), we get 
\begin{equation*}
\resizebox{0.99\hsize}{!}{$
\begin{aligned}
\vphantom{\frac{C^2}{2}}\cL(\Bar{x}_k)
        \leq & \cL(\Bar{x}_{k-1}) -(\alpha\rho - \alpha^2 L') \normfro{\Lambda(\Bar{x}_{k-1})}^2   + \frac{\eta \alpha}{2}\normfro{\nabla \cL (\Bar{x}_{k-1})}^2 \\
    \vphantom{\frac{C^2}{2}}    +&  L'\alpha^2\normfro{\Bar{y}_{k-1} -\Lambda(\Bar{x}_{k-1}) }^2 
        + \frac{\alpha}{2\eta}\normfro{\Lambda(\Bar{x}_{k-1}) - \Bar{y}_{k-1}}^2\\
        \vphantom{\frac{C^2}{2}} \leq & \cL(\Bar{x}_{k-1}) -(\alpha\rho - \alpha^2 L' - \eta \alpha C^2/2) \normfro{\Lambda(\Bar{x}_{k-1})}^2 \\
        \vphantom{\frac{C^2}{2}}+&  (L'\alpha^2+\frac{\alpha}{2\eta})\normfro{\Bar{y}_{k-1} -\Lambda(\Bar{x}_{k-1}) }^2, 
  \end{aligned}
  $}
\end{equation*}
where the last inequality is due to Lemma \ref{lem:bound_C}. We then apply Lemma \ref{lem:pseudo-grad-dominate} on the term $\normfro{\Lambda(\Bar{x}_{k-1})}^2$,
set $\eta = \frac{\rho}{2 C^2}$, consider $\alpha \leq \frac{\rho}{2L'}$, and use \eqref{eq:lipcont} on the last term to get the final result.



\qed

\noindent
\textbf{Proof of Lemma \ref{lem:bound_y_PL}:} We first recall \eqref{eq:y_consensus_bound_w_yk1}, showing that the consensus error on $\Y_k$ can be bounded using terms related to previous iterations, including $\normfro{ \Y_{k-1}  }^2$ for which we provide the following decomposition,$$\normfro{ \Y_{k-1}  }^2 \leq 2 \normfro{ \Y_{k-1} -  \Bar{\Y}_{k-1} }^2 + 2n\normfro{ \Bar{y}_{k-1} }^2,$$
and observe that
\begin{align*}
        2n\normfro{ \Bar{y}_{k-1} }^2 \leq &  4n\normfro{ \Bar{y}_{k-1} - \Lambda(\Bar{x}_{k-1}) }^2 + 4n\normfro{ \Lambda(\Bar{x}_{k-1}) }^2\\
        \leq &  4 L'^2\normfro{ \X_{k-1} -  \Bar{\X}_{k-1} }^2 + 4n\normfro{ \Lambda(\Bar{x}_{k-1}) }^2\\
        \leq & 4 L'^2\normfro{ \X_{k-1} -  \Bar{\X}_{k-1} }^2 + 48nL'\rho^{-2}\cL(\Bar{x}_{k-1}),
\end{align*}
where last inequality is derived from Lemma \ref{lem:descent_lemma} assuming $ \cL_{\mathcal{S}}^*=0$ for notation convenience. Plugging the inequalities above into \eqref{eq:y_consensus_bound_w_yk1}, we get
\begin{align*}
    \normfro{\Y_k - \Bar{\Y}_k}^2 
    &\leq  
    \Big(\frac{1+\sigma_W^2}{2} + 4L'^2\alpha^2\frac{1 + \sigma_W^2}{1-\sigma_W^2}\Big)\normfro{\Y_{k-1} - \Bar{\Y}_{k-1}}^2 \\
    &+8L'^2\frac{1 + \sigma_W^2}{1-\sigma_W^2} 
    (1+\alpha^2 L'^2)
    \normfro{\X_{k-1} - \Bar{\X}_{k-1} }^2\\
    &+ \frac{1 + \sigma_W^2}{1-\sigma_W^2}\frac{96n \alpha^2L'^3}{\rho^2}\cL(\Bar{x}_{k-1}).
\end{align*}

\qed

\noindent
\textbf{Proof of Theorem \ref{thm:distributed_convergence}:} The matrix $M$ is constructed by recalling the definition of the state vector in \eqref{eq:state_vector},
\begin{equation*}
    \xi_k \triangleq \begin{bmatrix}
        \normfro{\Y_k -  \Bar{\Y}_k}^2/L',\ 
        L'\normfro{\X_k -  \Bar{\X}_k}^2,\ 
        n(\cL(\Bar{x}_k) - \cL_{\mathcal{S}}^*)
    \end{bmatrix}^\top,
\end{equation*}
and applying Lemmas \ref{lem:bound_x}, \ref{lem:dist_V}, \ref{lem:bound_y_PL} on the three terms. For the ease of notation, we define $\Theta \triangleq \frac{1 + \sigma_W^2}{1-\sigma_W^2}$ to get
\begin{equation*}
\resizebox{0.99\hsize}{!}{$
\begin{aligned}
        &M \triangleq
    \begin{bmatrix}
        \frac{1+\sigma_W^2}{2} + 4L'^2\alpha^2\Theta &
        8 \alpha^2 L'^2\Theta + 8\Theta &
        96 L'^2\alpha^2\Theta \frac{1}{\rho^2}\\
        \alpha^2 L'^2\Theta&
        \frac{1+\sigma_W^2}{2}&
        0\\
        0&
        \alpha^2L'^2 + \frac{\alpha L' C^2 }{\rho}&
        1  - \frac{\alpha \rho\mu'}{4}
    \end{bmatrix}.
\end{aligned}
$}
\end{equation*}

We then prove that the system converges with a linear rate $1  - \frac{\alpha \rho \mu'}{8}$. Suppose that we have a positive vector $\boldsymbol{\delta} = [\delta_1, \delta_2, \delta_3]^\top$, such that 
    \begin{equation*}
        M \boldsymbol{\delta} \leq (1  - \frac{\alpha \rho \mu'}{8}) \boldsymbol{\delta},
    \end{equation*}
with element-wise inequality. Then, in lieu of Lemma \ref{lem:spectral_rad}, the spectral radius of $M$ is upper bounded by $1  - \frac{\alpha \rho \mu'}{8}$. 

In the next few lines of the proof, we solve these three inequalities to obtain sufficient conditions on $\alpha$ and find such $\boldsymbol{\delta}$. We start from the third row to get
\begin{align*}
    \Big( \alpha^2L'^2 + \frac{\alpha L' C^2 }{\rho} \Big) \delta_2
        + \Big(1  - \frac{\alpha \rho\mu'}{4}\Big) \delta_3
        &\leq \Big(1  - \frac{\alpha \rho\mu' }{8} \Big) \delta_3,
\end{align*}
for which, given $\alpha \leq \frac{1}{2L'}$, a sufficient condition is 
\begin{align*}
    \Big(\frac{4 L'}{\rho\mu'} + \frac{ 8L' C^2  }{\rho^2\mu'}\Big) \delta_2
        &\leq \delta_3.
\end{align*}
We define $\Phi \triangleq \frac{4 L'}{\rho\mu'} + \frac{ 8L' C^2  }{\rho^2\mu'}$, and fix $\delta_2 = 1, \delta_3 = \Phi.$ Next, we look at the first equation
\begin{align*}
    \Big(\frac{1+\sigma_W^2}{2} + 4L'^2\alpha^2\Theta\Big) \delta_1 
        &+ \Big(8 \alpha^2 L'^2\Theta + 8\Theta\Big)\delta_2 \\
        + \Big(96 L'^2\alpha^2\Theta \frac{1}{\rho^2}\Big) \delta_3
        &\leq \Big(1  - \frac{\alpha \rho\mu'}{8}\Big)\delta_1.
\end{align*}
We choose $\alpha > 0$ such that $4L'^2\alpha^2\Theta \leq \frac{1-\sigma_W^2}{8}$ and $\frac{\alpha \rho \mu'}{8} \leq \frac{1-\sigma_W^2}{8}$. This puts constraint on $\alpha$ such that 
\begin{equation}\label{eq:cond1}
\alpha \leq \min \left\{\sqrt{\frac{1-\sigma_W^2}{32 L'^2 \Theta}}, \frac{1-\sigma_W^2}{\rho\mu'}\right\}.    
\end{equation}
Then, with $\delta_2 = 1, \delta_3 = \Phi$, a sufficient condition for the inequality is 
\begin{align*}
     \left(8 \alpha^2 L'^2\Theta + 8\Theta\right)
        + \left(96 L'^2\alpha^2\Theta \frac{\Phi}{\rho^2}\right) 
        &\leq \frac{1-\sigma_W^2}{4}\delta_1.
\end{align*}
Therefore, we fix 
\begin{equation*}
    \begin{aligned}
        \delta_1 =& \frac{32}{1-\sigma_W^2}\left( \alpha^2 L'^2\Theta + \Theta+12 L'^2\alpha^2\Theta \frac{\Phi}{\rho^2}\right).
    \end{aligned}
\end{equation*}
Finally, we look at the second equation using our choices of $\delta_1$ and $\delta_2$ to get,
\begin{equation*}
\begin{aligned}
         \alpha^4 L'^4\Theta^2(1 + \frac{12\Phi}{\rho^2}) + \alpha^2 L'^2\Theta^2 
        &\leq  \frac{(1-\sigma_W^2)^2}{128}.
\end{aligned}
\end{equation*}
A sufficient condition for the above inequality to hold is that
\begin{align*}
     \alpha^4 L'^4\Theta^2(1 + \frac{12\Phi}{\rho^2}) &\leq  \frac{(1-\sigma_W^2)^2}{256}\\
    \alpha^2 L'^2\Theta^2 
        &\leq  \frac{(1-\sigma_W^2)^2}{256}.
\end{align*}
This can be achieved when 
\begin{equation*}
    \alpha \leq \sqrt[4]{\frac{(1-\sigma_W^2)^2}{256 L'^4\Theta^2 (1 + \frac{12\Phi}{\rho^2})}} = \frac{\sqrt{1-\sigma_W^2}}{4L' \sqrt{\Theta} \sqrt[4]{1 + \frac{12\Phi}{\rho^2}}},
\end{equation*}
and
\begin{equation*}
    \alpha \leq  \frac{1-\sigma_W^2}{16  L'\Theta}.
\end{equation*}
Combining the above constraints on $\alpha$ with \eqref{eq:cond1} and simplifying the sufficient conditions, we obtain
\begin{equation*}
\begin{aligned}
    \alpha \leq \min & \left\{\frac{1-\sigma_W^2}{\rho\mu'} , \frac{\sqrt{1-\sigma_W^2}}{4L' \sqrt{\Theta} \sqrt[4]{1 + \frac{12\Phi}{\rho^2}}}, \frac{1-\sigma_W^2}{16  L'\Theta} \right\}.
\end{aligned}
    \end{equation*}
We further want $\alpha \leq \frac{\rho}{2 L'}$ to apply Lemma \ref{lem:dist_V}, which together with above completes the proof for step size requirement.

We now need to verify that $\Bar{x}_k \in \mathcal{D}(\mathcal{S}, \delta) \cap \St(d, r)^{\frac{\epsilon}{2}}$ for every iteration $k$. We know $\Bar{x}_k \in \St(d,r)^\frac{\epsilon}{2}$ is already satisfied due to the safe step size constraint. Recall that $\mathbf{v}$ is the eigenvector corresponding to the spectral radius of $M$, which has all positive elements since $M$ is non-negative and irreducible. 
We prove by induction that if $\Bar{x}_k \in \mathcal{D}(\mathcal{S}, \delta)$ and $\xi_k \leq \frac{1}{8}n \mu' \rho^2 \delta^2 \mathbf{v}$, then $\Bar{x}_{k+1} \in \mathcal{D}(\mathcal{S}, \delta)$ and $\xi_{k+1} \leq\frac{1}{8} n \mu' \rho^2 \delta^2\mathbf{v}$, where the inequality is element-wise and $\normone{\mathbf{v}}=1$. We have
\begin{align*}
    \text{dist}(\mathcal{S}, \Bar{x}_{k+1}) &\leq\text{dist}(\mathcal{S}, \Bar{x}_k) + \normfro{\Bar{x}_k - \Bar{x}_{k+1}}\\
    &=\text{dist}(\mathcal{S}, \Bar{x}_k) + \alpha \normfro{ \Bar{y}_{k}}\\
    &\leq\text{dist}(\mathcal{S}, \Bar{x}_k) + \alpha \normfro{ \Bar{y}_{k} - \Lambda(\Bar{x}_k)} + \alpha \normfro{\Lambda(\Bar{x}_k)}\\
    &\leq(1+\alpha L')\text{dist}(\mathcal{S}, \Bar{x}_k) + \alpha L'n^{-1/2} \normfro{ \X_{k} - \Bar{\X}_k}\\
    &\leq 1.25~\text{dist}(\mathcal{S}, \Bar{x}_k)+0.25 n^{-1/2} \normfro{ \X_{k} - \Bar{\X}_k},
\end{align*}
where we used $\alpha L'\leq \rho/2 \leq 0.25$, Lipschitz continuity of $\Lambda$, and \eqref{eq:lipcont} in the last two lines. Applying Lemma \ref{lem:KL-QG}, we know that 
$$\text{dist}(\mathcal{S}, \Bar{x}_k)^2 \leq \frac{4(\cL(\Bar{x}_k) - \cL_\mathcal{S}^\ast)}{\mu' \rho^2}\leq \frac{\delta^2}{2},$$
where the last inequality is due to the induction assumption of $\xi_k \leq \frac{1}{8}n \mu' \rho^2 \delta^2 \mathbf{v}$, which also guarantees that $n^{-1/2} \normfro{ \X_{k} - \Bar{\X}_k} \leq \frac{\delta}{2\sqrt{2}}$. Therefore, we have $\text{dist}(\mathcal{S}, \Bar{x}_{k+1}) \leq \delta$, i.e., $\Bar{x}_{k+1} \in \mathcal{D}(\mathcal{S}, \delta)$. It is also immediate that $$\xi_{k+1} \leq M\xi_k \leq \frac{1}{8} n \mu' \rho^2 \delta^2 M\mathbf{v} < \frac{1}{8} n \mu' \rho^2 \delta^2\mathbf{v}.$$ 

Therefore, given the initial conditions on $x_0$ and $\xi_{0}$, by induction $\Bar{x}_k  \in \mathcal{D}(\mathcal{S}, \delta) \cap \St(d, r)^{\frac{\epsilon}{2}}$  and $\xi_k \leq \frac{1}{8}n \mu' \rho^2 \delta^2 \mathbf{v}$ are guaranteed to be satisfied for all $k$. 
\qed

\bibliographystyle{IEEEtran}
\bibliography{ref}
\begin{IEEEbiography}
	[{\includegraphics[width=0.95in,height=1in,clip,keepaspectratio]{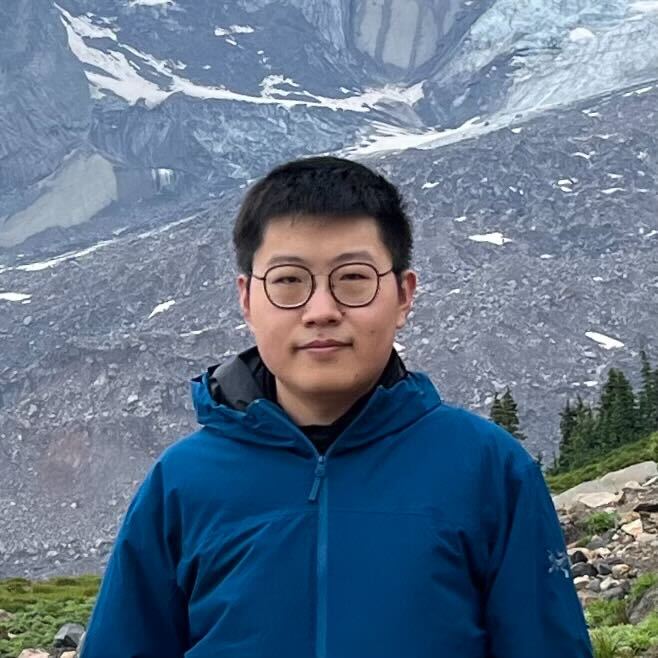}}] 
	{Youbang Sun}   is currently a Ph.D. candidate in the Department of Mechanical and Industrial Engineering at Northeastern University. His research interests include areas in machine learning and optimization with emphasis on distributed and multi-agent systems. He is interested in topics such as distributed optimization, Riemannian optimization, federated learning, and multi-agent reinforcement learning.
\end{IEEEbiography}

\begin{IEEEbiography}
	[{\includegraphics[width=0.95in,height=1in,clip,keepaspectratio]{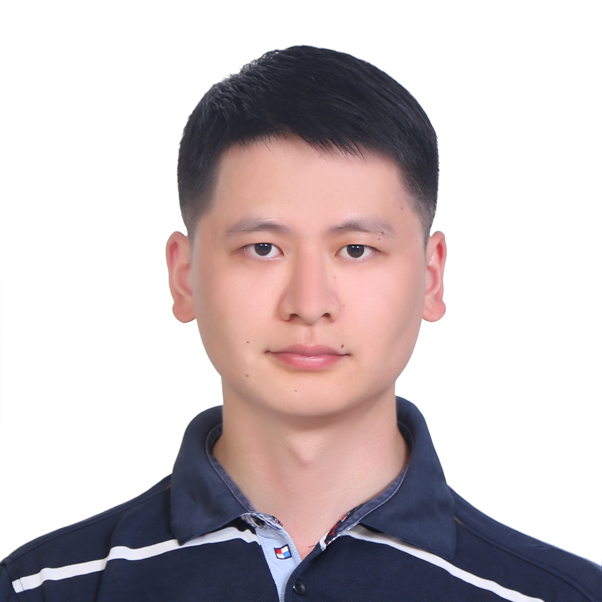}}] 
	{Shixiang Chen}   received the Ph.D. degree in Systems Engineering and Engineering Management from The Chinese University of Hong Kong in July, 2019.  He is   an assistant professor in the School of Mathematical Sciences, University of Science and Technology of China. He was a postdoctoral associate in the Department of Industrial \& Systems Engineering at Texas A\&M University.  His current research interests include design and analysis of optimization algorithms, and their applications in machine learning and signal processing.
\end{IEEEbiography}
\begin{IEEEbiography}
[{\includegraphics[width=0.95in,height=1.25in,clip,keepaspectratio]{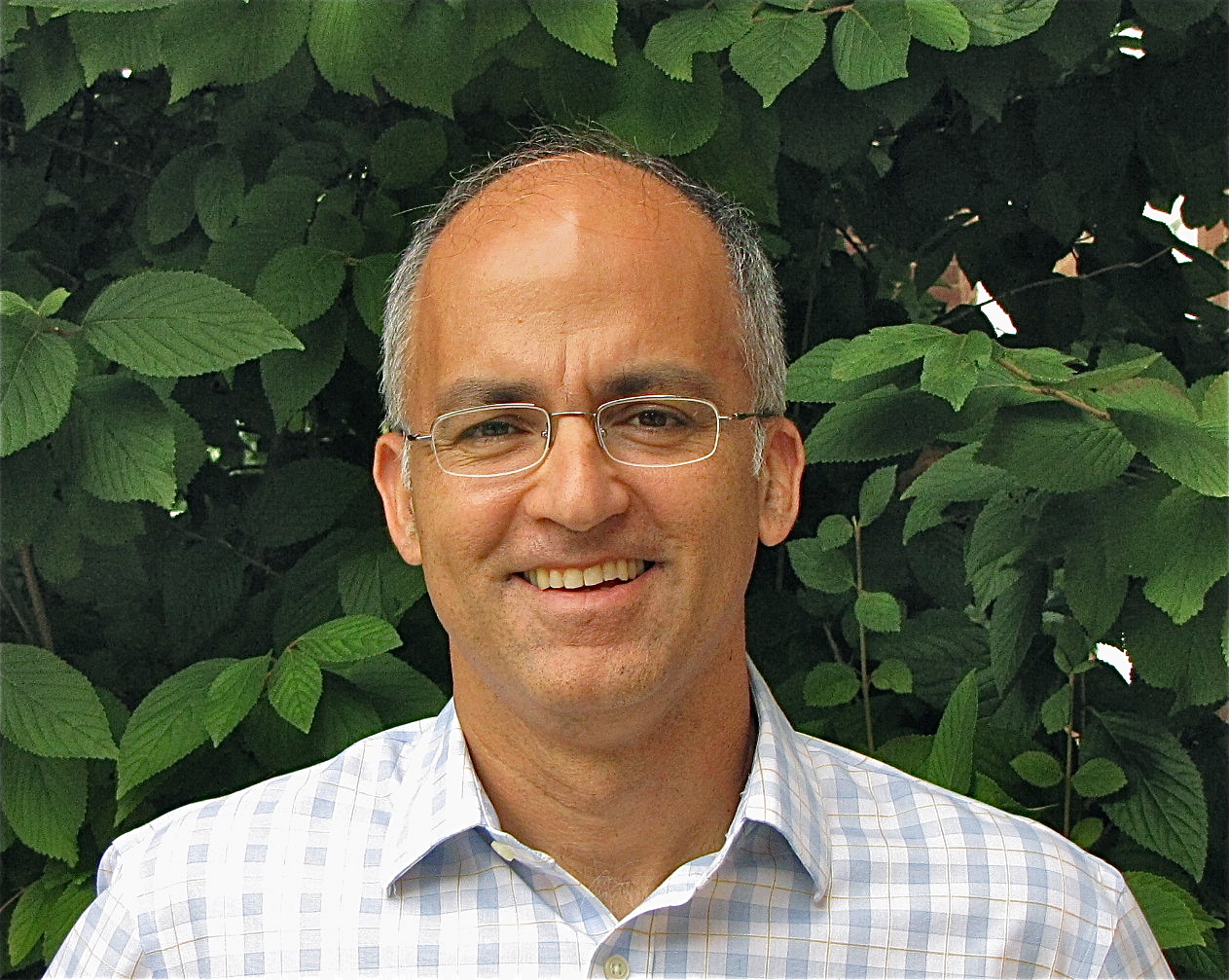}}] 
{Alfredo Garcia} received the Degree in electrical engineering from the Universidad de los Andes, Bogot\'a,Colombia,in 1991,the Dipl\^ome d’Etudes Approfondies in automatic control from the Universit\'e Paul Sabatier, Toulouse, France, in 1992, and the Ph.D. degree in industrial and operations engineering from the University of Michigan, Ann Arbor, MI, USA, in 1997.
From 1997 to 2001, he was a consultant to government agencies and private utilities in the electric power industry. From 2001 to 2015, he
was a Faculty with the Department of Systems and Information Engineering, University of Virginia, Charlottesville, VA, USA. From 2015 to 2017, he was a Professor with the Department of Industrial and Systems Engineering, University of Florida, Gainesville, FL, USA. In 2018, he joined the Department of Industrial and System Engineering, Texas A\&M University, College Station, TX, USA. His research interests include game theory and dynamic optimization with applications in communications and energy networks.
\end{IEEEbiography}

\begin{IEEEbiography}
	[{\includegraphics[width=1.35in,height=1.1in,clip,keepaspectratio]{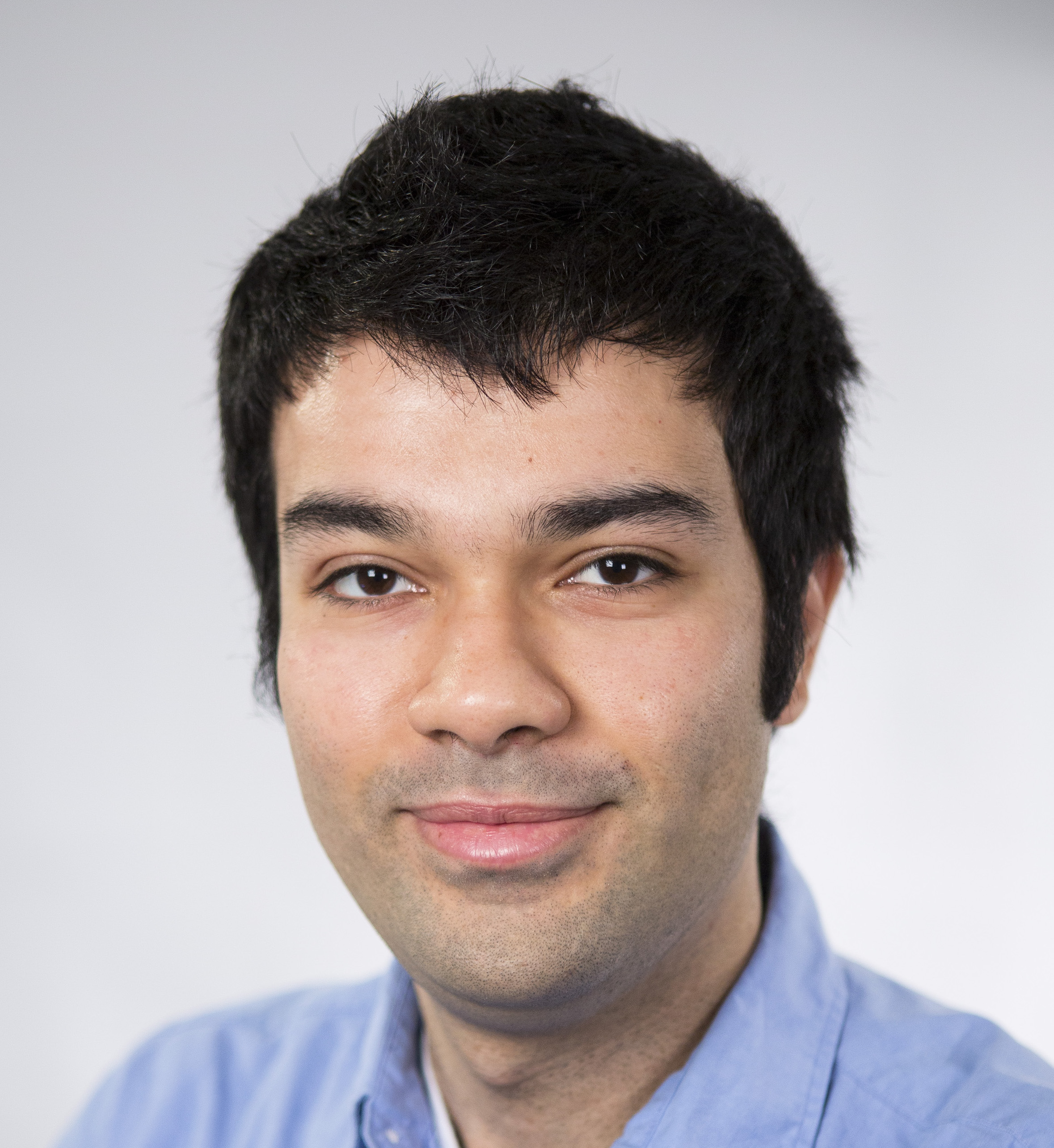}}] {Shahin Shahrampour} received the Ph.D. degree in Electrical and Systems Engineering, the M.A. degree in Statistics (The Wharton School), and the M.S.E. degree in Electrical Engineering, all from the University of Pennsylvania, in 2015, 2014, and 2012, respectively. He is currently an Assistant Professor in the Department of Mechanical and Industrial Engineering at Northeastern University. His research interests include machine learning, optimization, sequential decision-making, and distributed learning, with a focus on developing computationally efficient methods for data analytics. He is a Senior Member of the IEEE.
	\end{IEEEbiography}

\end{document}